%% file: group_diagnosis_v7.tex
\documentclass[letterpaper,11pt]{article}
\input{commands}

\usepackage{graphicx}
\usepackage{subfigure}
\usepackage{pifont}
\usepackage{float}
\usepackage[ruled,vlined]{algorithm2e}
\usepackage[all]{xy}
\usepackage{url}
\usepackage{fullpage}
\usepackage{setspace}
\usepackage{amsthm, amsmath, amssymb}
\usepackage{stmaryrd}
\usepackage[lmargin = 0.75in, rmargin = 0.75in, tmargin = 0.75in, bmargin = 0.7in]{geometry}

\newtheorem{thm}{Theorem}

\newtheorem{cor}{Corollary}
\newtheorem{lemma}{Lemma}
\newtheorem{defn}{Definition}

\theoremstyle{plain}

\newcommand\prob{\mathop{\rm Pr}}
\newcommand\E{\mathbb{E}}
\newcommand\T{T}
\newcommand\cT{\mathcal{T}}

\newcommand\cP{\Pi}
\newcommand{\cPt}{\widetilde{\Pi}}
\newcommand\cI{\mathcal{I}}
\newcommand\cL{\mathcal{L}}

\newcommand\vp{\mathbf{p}}
\newcommand\vy{\mathbf{y}}
\newcommand\vz{\mathbf{z}}
\newcommand\vb{\mathbf{b}}
\newcommand\vx{\mathbf{x}}

\newcommand{\B}{\mathbf{B}}
\newcommand{\Bt}{\widetilde{\mathbf{B}}}

\newcommand{\p}{\tilde{p}}

\begin{document}

\title{\textbf{Group-based query learning \\ for rapid diagnosis in time-critical situations}}

\date{}
\author{\small Gowtham Bellala$^{1}$, Suresh K. Bhavnani$^{2}$, Clayton D. Scott$^{1,2}$ \\ \small $^{1}$Department of Electrical Engineering and Computer Science,\\ \small $^2$Center for Computational Medicine and Bioinformatics, University of Michigan, Ann Arbor, MI 48109 \\ \small E-mail: \{ gowtham, bhavnani, clayscot \}@umich.edu}
\maketitle

\begin{abstract}
In query learning, the goal is to identify an unknown object while minimizing the number of "yes or no" questions (queries) posed about that object. We consider three extensions of this fundamental problem that are motivated by practical considerations in real-world, time-critical identification tasks such as emergency response.  First, we consider the problem where the objects are partitioned into groups, and the goal is to identify only the group to which the object belongs. Second, we address the situation where the queries are partitioned into groups, and an algorithm may suggest a group of queries to a human user, who then selects the actual query. Third, we consider the problem of query learning in the presence of persistent query noise, and relate it to group identification. To address these problems we show that a standard algorithm for query learning, known as the splitting algorithm or generalized binary search, may be viewed as a generalization of Shannon-Fano coding. We then extend this result to the group-based settings, leading to new algorithms. The performance of our algorithms is demonstrated on simulated data and on a database used by first responders for toxic chemical identification.
\end{abstract}

%

\section{Introduction}

In emergency response applications, as well as other time-critical diagnostic tasks, there is a need to rapidly identify a cause by selectively acquiring information from the environment. For example, in the problem of toxic chemical identification, a first responder may question victims of chemical exposure regarding the symptoms they experience. Chemicals that are inconsistent with the reported symptoms may then be eliminated. Because of the importance of this problem, several organizations have constructed extensive evidence-based databases (e.g., Haz-Map\footnote{\url{http://hazmap.nlm.nih.gov/}}) that record toxic chemicals and the acute symptoms which they are known to cause. Unfortunately, many symptoms tend to be nonspecific (e.g., vomiting can be caused by many different chemicals), and it is therefore critical for the first responder to pose these questions in a sequence that leads to chemical identification in as few questions as possible.

This problem has been studied from a mathematical perspective for decades, and has been described variously as query learning (with membership queries) \cite{angluin}, active learning \cite{dasgupta}, object/entity identification \cite{garey1,garey2}, and binary testing \cite{garey2,loveland}. In this work we refer to the problem as query learning or object identification. The standard mathematical formulation of query learning is often idealized relative to many real-world diagnostic tasks, in that it does not account for time constraints and resulting input errors. In this paper we investigate algorithms that extend query learning to such more realistic settings by addressing the need for rapid response, and error-tolerant algorithms.


In query learning there is an unknown object $\theta$ belonging to a set $\Theta = \{\theta_1,\cdots,\theta_M\}$ of $M$ objects and a set $Q = \{q_1,\cdots,q_N\}$ of $N$ distinct subsets of $\Theta$ known as queries. Additionally, the vector $\cP = (\pi_1,\cdots,\pi_M)$ denotes the \emph{a priori} probability distribution over $\Theta$. The goal is to determine the unknown object $\theta \in \Theta$ through as few queries from $Q$ as possible, where a query $q \in Q$ returns a value $1$ if $\theta \in q$, and $0$ otherwise. A query learning algorithm thus corresponds to a decision tree, where the internal nodes are queries, and the leaf nodes are objects. Problems of this nature arise in applications such as fault testing \cite{koren,unluyurt}, machine diagnostics \cite{shiozaki}, disease diagnosis \cite{loveland,pattipati}, computer vision \cite{geman} and active learning \cite{dasgupta,nowak}. Algorithms and performance guarantees have been extensively developed in the literature, as described in Section \ref{sec:related work} below.

In the context of toxic chemical identification, the objects are chemicals, and the queries are symptoms. A query learning algorithm will prompt the first responder with a symptom. Once the presence or absence of that symptom is determined, a new symptom is suggested by the algorithm, and so on, until the chemical is uniquely determined. In this paper, we consider variations on this basic query learning framework that are motivated by toxic chemical identification, and are naturally applicable to many other time-critical diagnostic tasks. In particular, we develop theoretical results and new algorithms for what might be described as group-based query learning.

First, we consider the case where $\Theta$ is partitioned into groups of objects, and it is only necessary to identify the group to which the object belongs. For example, the appropriate response to a toxic chemical may only depend on the class of chemicals to which it belongs (pesticide, corrosive acid, etc.). As our experiments reveal, a query learning algorithm designed to rapidly identify individual objects is not necessarily efficient for group identification.

Second, we consider the problem where the set $Q$ of queries is partitioned into groups (respiratory symptoms, cardio symptoms, etc.). Instead of suggesting specific symptoms to the user, we design an algorithm that suggests a group of queries, and allows the user the freedom to input information on any query in that group. Although such a system will theoretically be less efficient, it is motivated by the fact that in a practical application, some symptoms will be easier for a given user to understand and identify. Instead of suggesting a single symptom, which might seem ''out of the blue" to the user, suggesting a query group will be less bewildering, and hence lead to a more efficient and accurate outcome. Our experiments demonstrate that the proposed algorithm based on query groups identifies objects in nearly as few queries as a fully active method.

Third, we apply our algorithm for group identification to the problem of query learning with persistent query noise. Persistent query noise occurs when the response of a query is in error, but cannot be resampled, as is often assumed in the literature. Such is the case when the presence or absence of a symptom is incorrectly determined, which is more likely in a stressful emergency response scenario. Experiments show our method offers significant gains over algorithms not designed for persistent query noise.

Our algorithms are derived in a common framework, and are based on a reinterpretation of a standard query learning algorithm (the splitting algorithm, or generalized binary search) as a generalized form of Shannon-Fano coding. We first establish an exact formula for the expected number of queries by an arbitrary decision tree, and show that the splitting algorithm effectively performs a greedy, top-down optimization of this objective. We then extend this formula to the case of group identification and query groups, and develop analogous greedy algorithms. In the process, we provide a new interpretation of impurity-based decision tree induction for multi-class classification. 

We apply our algorithms to both synthetic data and to the WISER database (version $4.21$). WISER\footnote{\url{http://wiser.nlm.nih.gov/}}, which stands for {\bf W}ireless {\bf I}nformation {\bf S}ystem for {\bf E}mergency {\bf R}esponders, is a decision support system developed by the National Library of Medicine (NLM) for first responders. This database describes the binary relationship between $298$ toxic chemicals (corresponds to the number of distinguishable chemicals in this database) and $79$ acute symptoms. The symptoms are grouped into $10$ categories (e.g., neurological, cardio) as determined by NLM, and the chemicals are grouped into $16$ categories (e.g., pesticides, corrosive acids) as determined by a toxicologist and a Hazmat expert. 

\subsection{Prior and related work}
\label{sec:related work}
The problem of selecting an optimal sequence of queries from $Q$ to uniquely identify the unknown object $\theta$ is equivalent to determining an optimal binary decision tree, where each internal node in the tree corresponds to a query, each leaf node corresponds to a unique object from the set $\Theta$ and the optimality is with respect to minimizing the expected depth of the leaf node corresponding to $\theta$. In the special case when the query set $Q$ is \emph{complete} (where a query set $Q$ is said to be \emph{complete} if for any $S \subseteq \Theta$ there exists a query $q \in Q$ such that either $q = S$ or $\Theta \setminus q = S$), the problem of constructing an optimal binary decision tree is equivalent to construction of optimal variable-length binary prefix codes with minimum expected length. This problem has been widely studied in information theory with both Shannon \cite{shannon} and Fano \cite{fano} independently proposing a top-down greedy strategy to construct suboptimal binary prefix codes, popularly known as Shannon-Fano codes. Later Huffman \cite{huffman} derived a simple bottom-up algorithm to construct optimal binary prefix codes. A well known lower bound on the expected length of binary prefix codes is given by the Shannon entropy of the probability distribution $\cP$ \cite{cover}. 

When the query set $Q$ is not \emph{complete}, a query learning problem can be considered as ``constrained'' prefix coding with the same lower bound on the expected depth of a tree. This problem has also been studied extensively in the literature with Garey \cite{garey1,garey2} proposing a dynamic programming based algorithm to find an optimal solution. This algorithm runs in exponential time in the worst case. Later, Hyafil and Rivest \cite{rivest} showed that determining an optimal binary decision tree for this problem is NP-complete. Thereafter, various greedy algorithms \cite{loveland,roy,kosaraju} have been proposed to obtain a suboptimal binary decision tree. The most widely studied algorithm known as the \emph{splitting algorithm} \cite{loveland} or \emph{generalized binary search} (GBS) \cite{dasgupta,nowak}, selects a query that most evenly divides the probability mass of the remaining objects \cite{dasgupta,loveland,nowak,goodman}. Various bounds on the performance of this greedy algorithm have been established in \cite{dasgupta,loveland,nowak}. Goodman and Smyth \cite{goodman} observe that this algorithm can be viewed as a generalized version of Shannon-Fano coding. In Section \ref{sec:object identification}, we demonstrate the same through an alternative approach that can be generalized to group-based query learning problems, leading to efficient algorithms in these settings. As far as we know, there has been no previous work on group queries or group identification.  

Though most of the above work has been devoted to query learning in the ideal setting assuming no noise, it is unrealistic to assume that the responses to queries are without error in many applications. The problem of learning in the presence of query noise has been studied in \cite{nowak,kaar,nowak2} where the queries can be resampled or repeated. However, in certain applications, resampling or repeating the query does not change the query response confining the algorithm to non-repeatable queries. The work by R\'{e}nyi in \cite{renyi} is regarded to be the first to consider this more stringent noise model, also referred to as persistent noise in the literature \cite{goldman,jackson,hanneke}. However, his work has focused on the passive setting where the queries are chosen at random. Learning under persistent noise model has also been studied in \cite{goldman,jackson,angluin2} where the goal was to identify or learn Disjunctive Normal Form (DNF) formulae from noisy data. The query (label) complexity of pool-based active learning in the Probably Approximately Correct (PAC) model in the presence of persistent classification noise has been studied in \cite{hanneke} and active learning algorithms in this setting have been proposed in \cite{hanneke,balcan}. Here, we focus on the problem of query learning under the persistent noise model where the goal is to uniquely identify the true object. Finally, this work was motivated by earlier work that applied GBS to WISER \cite{suresh}.   

\subsection{Notation}
\label{sec:notation}
We denote a query learning problem by a pair $(\B,\cP)$ where $\B$ is a binary matrix with $b_{ij}$ equal to $1$ if $\theta_i \in q_j$, and $0$ otherwise. A decision tree $\T$ constructed on $(\B,\cP)$ has a query from the set $Q$ at each of its internal nodes with the leaf nodes terminating in the objects from the set $\Theta$. At each internal node in the tree, the object set under consideration is divided into two subsets, corresponding to the objects that respond $0$ and $1$ to the query, respectively. For a decision tree with $L$ leaves, the leaf nodes are indexed by the set $\cL = \{1,\cdots,L\}$ and the internal nodes are indexed by the set $\cI = \{L+1,\cdots,2L-1\}$. At any internal node $a \in \cI$, let $l(a), r(a)$ denote the ``left'' and ``right'' child nodes, where the set $\Theta_a \subseteq \Theta$ corresponds to the set of objects that reach node `$a$', and the sets $\Theta_{l(a)} \subseteq \Theta_a, \Theta_{r(a)} \subseteq \Theta_a$ corresponds to the set of objects that respond $0$ and $1$ to the query at node `$a$', respectively. We denote by $\pi_{\Theta_a} := \sum_{\{i: \theta_i \in \Theta_a \}}\pi_i$, the probability mass of the objects under consideration at any node `$a$' in the tree. Also, at any node `$a$', the set $Q_a \subseteq Q$ corresponds to the set of queries that have been performed along the path from the root node up to node `$a$'. 

We denote the Shannon entropy of a vector $\cP = (\pi_1,\cdots,\pi_M)$ by $H(\cP) := - \sum_i \pi_i \log_2 \pi_i$ and the Shannon entropy of a proportion $\pi \in [0,1]$ by $H(\pi) := -\pi \log_2 \pi - (1-\pi) \log_2 (1-\pi)$, where we use the limit, $\displaystyle \lim_{\pi \rightarrow 0} \pi \log_2 \pi = 0$ to define the limiting cases. Finally, we use the random variable $K$ to denote the number of queries required to identify an unknown object $\theta$ or the group of an unknown object $\theta$ using a given tree.

\section{Generalized Shannon-Fano Coding}
\label{sec:object identification}
Before proceeding to group-based query learning, we first present an exact formula for standard query learning problems. This result allows us to interpret the splitting algorithm or GBS as generalized Shannon-Fano coding. Furthermore, our proposed algorithms for group-based settings are based on generalizations of this result. 

First, we define a parameter called the \emph{reduction factor} on the binary matrix/tree combination that provides a useful quantification on the expected number of queries required to identify an unknown object.

\begin{defn}
\label{defn:rhor}
A \emph{reduction factor} at any internal node `$a$' in a decision tree is defined as \\ 
$\rho_a = \max(\pi_{\Theta_{l(a)}},\pi_{\Theta_{r(a)}})/\pi_{\Theta_{a}}$ and the \emph{overall reduction factor} of a tree is defined as $\rho = \max_{a \in \cI}\rho_a$.
\end{defn}
Note from the above definition that $0.5 \leq \rho_a \leq \rho \leq 1$ and we describe a decision tree with $\rho = 0.5$ to be a perfectly balanced tree. 

Given a query learning problem $(\B,\cP)$, let $\cT(\B,\cP)$ denote the set of decision trees that can uniquely identify all the objects in the set $\Theta$. For any decision tree $\T \in \cT(\B,\cP)$, let $\{\rho_a\}_{a \in \cI}$ denote the set of reduction factors and let $d_i$ denote the depth of object $\theta_i$ in the tree. Then the expected number of queries required to identify an unknown object using the given tree is equal to
\begin{eqnarray*}
\E[K] = \sum_{i=1}^M \prob(\theta = \theta_i) \E[K|\theta = \theta_i]  = \sum_{i=1}^M \pi_i d_i .
\end{eqnarray*}

\begin{thm}
\label{thm:object identification}
The expected number of queries required to identify an unknown object using a tree $\T$ with reduction factors $\{\rho_a \}_{a \in \cI}$ constructed on $(\B,\cP)$ is given by
\begin{equation}
\label{eq:object identification}
\E[K] = H(\cP) + \sum_{a \in \cI} \pi_{\Theta_{a}}[1 - H(\rho_a)] = \frac{H(\cP)}{\sum_{a\in \cI}\tilde{\pi}_{\Theta_a}H(\rho_a)}
\end{equation}
where $\tilde{\pi}_{\Theta_a} := \frac{\pi_{\Theta_a}}{\sum_{r\in \cI}\pi_{\Theta_r}}$.
\end{thm}
\begin{proof}
The first equality is a special case of Theorem \ref{thm:group identification} below. The second equality follows from the observation $\E[K] = \sum_{i=1}^M \pi_i d_i = \sum_{a \in \cI} \pi_{\Theta_a}$. Hence replacing $\pi_{\Theta_a}$ with $\tilde{\pi}_{\Theta_a} \cdot \E[K]$ in the first equality leads to the result.
\end{proof}

In the second equality, the term $\sum_{a \in \cI} \tilde{\pi}_{\Theta_a}H(\rho_a)$ denotes the average entropy of the reduction factors, weighted by the proportion of times each internal node `$a$' is queried in the tree. This theorem re-iterates an earlier observation that the expected number of queries required to identify an unknown object using a tree constructed on $(\B,\cP)$ (where the query set $Q$ is not necessarily a \emph{complete} set) is bounded below by its entropy $H(\cP)$. It also follows from the above result that a tree attains this minimum value (i.e., $\E[K] = H(\cP)$) iff it is perfectly balanced, i.e., the overall reduction factor $\rho$ of the tree is equal to $0.5$. 

From the first equality, the problem of finding a decision tree with minimum $\E[K]$ can be formulated as the following optimization problem

\begin{eqnarray}
\label{eq:optimization object identification}
& \underset{\T \in \cT(\B,\cP)}{\operatorname{\min}} H(\cP ) + \sum_{a \in \cI} \pi_{\Theta_a}[1 - H(\rho_a )] &
\end{eqnarray}
Since $\cP$ is fixed, the optimization problem reduces to minimizing $\sum_{a \in \cI} \pi_{\Theta_a}[1 - H(\rho_a)]$ over the set of trees $\cT(\B,\cP)$. Note that the reduction factor $\rho_a$ depends on the query chosen at node `$a$' in a tree $\T$. As mentioned earlier, finding a global optimal solution for this optimization problem is NP-complete. 

Instead, we may take a top down approach and minimize the objective function by minimizing the term $\pi_{\Theta_a} [1 - H(\rho_a)]$ at each internal node, starting from the root node. Since $\pi_{\Theta_a}$ is independent of the query chosen at node `$a$', this reduces to minimizing $\rho_a$ (i.e., choosing a split as balanced as possible) at each internal node $a \in \cI$. The algorithm can be summarized as shown in Algorithm \ref{algo_objectidentification} below.

\restylealgo{boxed}
\begin{algorithm}
\dontprintsemicolon
\caption{Greedy decision tree algorithm for object identification \label{algo_objectidentification}}
\textbf{\underline{Generalized Binary Search (GBS)}} \;
\BlankLine
\textbf{Initialization :}  \emph{Let the leaf set consist of the root node} \;
\SetLine
\While{some leaf node `$a$' has $|\Theta_a| > 1$}{
\For{each query $q \in Q \setminus Q_a$}{
Find $\Theta_{l(a)}$ and $\Theta_{r(a)}$ produced by making a split with query $q$\;
Compute the reduction factor $\rho_a$ produced by query $q$\;
}
Choose a query with the smallest reduction factor\;
Form child nodes $l(a),r(a)$\;
}
\end{algorithm}
Note that when the query set $Q$ is \emph{complete}, Algorithm \ref{algo_objectidentification} is similar to Shannon-Fano coding \cite{shannon,fano}. The only difference is that in Shannon-Fano coding, for computational reasons, the queries are restricted to those that are based on thresholding the prior probabilities $\pi_i$.

\begin{cor}
The standard splitting algorithm/GBS is a greedy algorithm to minimize the expected number of queries required to uniquely identify an object.
\end{cor}

Corollary \ref{cor:object identification} below follows from Theorem \ref{thm:object identification}. It states that given a tree $\T$ with overall reduction factor $\rho < 1$, the average complexity of identifying an unknown object using this tree is $O(\log_2 M)$. Recently, Nowak \cite{nowak} showed there are geometric conditions (incoherence and neighborliness) that also bound the worst-case depth of the tree to be $O(\log_2 M)$, assuming a uniform prior on objects.  The conditions imply that the reduction factors are close to $\frac12$ except possibly near the very bottom of the tree where they could be close to 1. 

\begin{cor}
\label{cor:object identification}
The expected number of queries required to identify an unknown object using a tree with overall reduction factor $\rho$ constructed on $(\B,\cP)$ is bounded above by
$$
\E[K] \le \frac{H(\cP)}{H(\rho)} \le \frac{\log_2 M}{H(\rho)}
$$
\end{cor}
\begin{proof}
Using the second equality in Theorem \ref{thm:object identification}, we get
\begin{eqnarray}
\label{eq:mu cor}
\E[K]  =  \frac{H(\cP )}{\sum_{a\in \cI}\tilde{\pi}_{\Theta_a}H(\rho_a)} 
 \leq  \frac{H(\cP)}{H(\rho)}  
 \leq  \frac{\log_2M}{H(\rho)} \nonumber
\end{eqnarray}
where the first inequality follows from the definition of $\rho$, $\rho \geq \rho_a \geq 0.5 , \forall a \in \cI$ and the last inequality follows from the concavity of the entropy function.
\end{proof}

In the sections that follow, we show how Theorem \ref{thm:object identification} and Algorithm \ref{algo_objectidentification} may be generalized, leading to principled strategies for group identification, query learning with group queries and query learning with persistent noise.

\section{Group Identification}
\label{sec:group identification}
We now move to the problem of group identification, where the goal is not to determine the object, but only the group to which the object belongs. Here, in addition to the binary matrix $\B$ and \emph{a priori} probability distribution $\cP$ on the objects, the group labels for the objects are also provided, where the groups are assumed to be disjoint.

We denote a query learning problem for group identification by $(\B,\cP,\vy)$, where $\vy = (y_1,\cdots,y_M)$ denotes the group labels of the objects, $y_i \in \{1,\cdots,m\}$. Let $\{\Theta^i\}_{i=1}^m$ be a partition of the object set $\Theta$, where $\Theta^i$ denotes the set of objects in $\Theta$ that belong to group $i$. It is important to note here that the group identification problem cannot be simply reduced to a standard query learning problem with groups $\{\Theta^1,\cdots,\Theta^m\}$ as meta ``objects,'' since the objects within a group need not respond the same to each query. For example, consider the toy example shown in Figure \ref{fig:toy network} where the objects $\theta_1,\theta_2$ and $\theta_3$ belonging to group $1$ cannot be considered as one single meta object as these objects respond differently to queries $q_1$ and $q_3$. 

In this context, we also note that GBS can fail to find a good solution for a group identification problem as it does not take the group labels into consideration while choosing queries. Once again, consider the toy example shown in Figure \ref{fig:toy network} where just one query (query $q_2$) is sufficient to identify the group of an unknown object, whereas GBS requires $2$ queries to identify the group when the unknown object is either $\theta_2$ or $\theta_4$, as shown in Figure \ref{fig:toy tree}. Hence, we develop a new strategy which accounts for the group labels when choosing the best query at each stage. 

\begin{figure*}[!t]
\renewcommand{\arraystretch}{1.3}
\begin{minipage}[b]{0.5\linewidth}
\centering
\begin{tabular}{|c|c c c|c|}
\hline
& $q_1$ & $q_2$ & $q_3$ & Group label, $y$ \\
\hline
$\theta_1$ & 0 & 1 & 1 & 1 \\
$\theta_2$ & 1 & 1 & 0 & 1 \\
$\theta_3$ & 0 & 1 & 0 & 1 \\
$\theta_4$ & 1 & 0 & 0 & 2 \\
\hline
\end{tabular}
\caption{\small \sl Toy Example 1}
\label{fig:toy network}
\end{minipage}
\hspace{0.15cm}
\begin{minipage}[b]{0.5\linewidth}
\centering
\begin{displaymath}
\xymatrix{
  &  *++[o][F-]{q_1} \ar[dl]_0 \ar[dr]^1 & &  \\
y = 1 & & *++[o][F-]{q_2} \ar[dl]_0 \ar[dr]^1 & \\
& y = 2 & & y = 1 \\
}
\end{displaymath}
\caption{\small \sl Decision tree constructed using GBS for group identification on toy example 1}
\label{fig:toy tree}
\end{minipage}
\end{figure*}

Note that when constructing a tree for group identification, a greedy, top-down algorithm terminates splitting when all the objects at the node belong to the same group. Hence, a tree constructed in this fashion can have multiple objects ending in the same leaf node and multiple leaves ending in the same group. 

For a tree with $L$ leaves, we denote by $\cL^i \subset \cL = \{1,\cdots,L\}$ the set of leaves that terminate in group $i$. Similar to $\Theta^i \subseteq \Theta$, we denote by $\Theta_a^i \subseteq \Theta_a$ the set of objects that belong to group $i$ at any internal node $a \in \cI$ in the tree.  Also, in addition to the reduction factors defined in Section \ref{sec:object identification}, we define a new set of reduction factors called the group reduction factors at each internal node.

\begin{defn}
\label{defn:grhor}
The group reduction factor of group $i$ at any internal node `$a$' in a decision tree is defined as
$\rho_a^i = \max(\pi_{\Theta_{l(a)}^i},\pi_{\Theta_{r(a)}^i})/\pi_{\Theta_{a}^i}$.
\end{defn}

Given $(\B,\cP,\vy)$, let $\cT(\B,\cP,\vy)$ denote the set of decision trees that can uniquely identify the groups of all objects in the set $\Theta$. For any decision tree $\T \in \cT(\B,\cP,\vy)$, let $\rho_a$ denote the reduction factor and let $\{\rho_a^i\}_{i=1}^m$ denote the set of group reduction factors at each of its internal nodes. Also, let $d_j$ denote the depth of leaf node $j \in \cL$ in the tree. Then the expected number of queries required to identify the group of an unknown object using the given tree is equal to 
\begin{eqnarray*}
\E[K] & = & \sum_{i=1}^m \prob(\theta \in \Theta^i) \E[K |\theta \in \Theta^i] \\
& = & \sum_{i=1}^m \pi_{\Theta^i} \left [ \sum_{j \in \cL^i} \frac{\pi_{\Theta_j}}{\pi_{\Theta^i}} d_j \right ] 
\end{eqnarray*}

\begin{thm}
\label{thm:group identification}
The expected number of queries required to identify the group of an object using a tree $\T$ with reduction factors $\{\rho_a \}_{a \in \cI}$ and group reduction factors $\{\rho_a^i\}_{i=1}^m, \forall a \in \cI$ constructed on $(\B,\cP,\vy)$, is given by
\begin{equation}
\label{eq:relation group identification}
\E[K] = H(\cP_{\vy} ) + \sum_{a \in \cI} \pi_{\Theta_{a}} \left [1 - H(\rho_a) + \sum_{i=1}^m \frac{\pi_{\Theta_a^i}}{\pi_{\Theta_a}} H(\rho_a^i) \right ] 
\end{equation}
where $\cP_{\vy}$ denotes the probability distribution of the object groups induced by the labels $\vy$, i.e. $\cP_{\vy} = (\pi_{\Theta^1},\cdots,\pi_{\Theta^m})$.
\end{thm}
\begin{proof}
Special case of Theorem \ref{thm:group identification group queries} below.
\end{proof}

The above theorem states that given a query learning problem for group identification $(\B,\cP,\vy)$, the expected number of queries required to identify the group of an unknown object is lower bounded by the entropy of the probability distribution of the groups. It also follows from the above result that this lower bound is achieved iff there exists a perfectly balanced tree (i.e. $\rho = 0.5$) with the group reduction factors equal to $1$ at every internal node in the tree. Also, note that Theorem \ref{thm:object identification} is a special case of this theorem where each group has size $1$ leading to $\rho_a^i = 1$ for all groups at every internal node. 

Using Theorem \ref{thm:group identification}, the problem of finding a decision tree with minimum $\E[K]$ can be formulated as the following optimization problem
\begin{eqnarray}
\label{eq:optimization group identification}
& \underset{\T \in \cT(\B,\cP,\vy)}{\operatorname{\min}}  \sum_{a \in \cI} \pi_{\Theta_a}\left [1 - H(\rho_a) + \sum_{i=1}^m \frac{\pi_{\Theta_a^i}}{\pi_{\Theta_a}} H(\rho_a^i) \right ]  &
\end{eqnarray}

Note that here both the reduction factor $\rho_a$ and the group reduction factors $\{ \rho_a^i\}_{i=1}^m$ depend on the query chosen at node `$a$'. Also, the above optimization problem being a generalized version of the optimization problem in (\ref{eq:optimization object identification}) is NP-complete. Hence, we propose a suboptimal approach to solve the above optimization problem where we optimize the objective function locally instead of globally. We take a top-down approach and minimize the objective function by minimizing the term $C_a := \left [ 1 - H(\rho_a) + \sum_{i=1}^m \frac{\pi_{\Theta_a^i}}{\pi_{\Theta_a}} H(\rho_a^i) \right ]$ at each internal node, starting from the root node. The algorithm can be summarized as shown in Algorithm \ref{algo_groupidentification} below. This algorithm is referred to as GISA (Group Identification Splitting Algorithm) in the rest of this paper.

\restylealgo{boxed}
\begin{algorithm}
\dontprintsemicolon
\caption{Greedy decision tree algorithm for group identification \label{algo_groupidentification}}

\textbf{\underline{Group Identification Splitting Algorithm (GISA)}} \;
\BlankLine
\textbf{Initialization :}  \emph{Let the leaf set consist of the root node} \;
\SetLine
\While{some leaf node `$a$' has more than one group of objects}{
\For{each query $q_j \in Q \setminus Q_a$}{
Compute $\{\rho_{a}^i\}_{i=1}^m$ and $\rho_a$ produced by making a split with query $q_j$\;
Compute the cost $C_a(j)$ of making a split with query $q_j$\;
}
Choose a query with the least cost $C_a$ at node `$a$'\;
Form child nodes $l(a),r(a)$\;
}
\end{algorithm}

Note that the objective function in this algorithm consists of two terms. The first term $[1 - H(\rho_a)]$ favors queries that evenly distribute the probability mass of the objects at node `$a$' to its child nodes (regardless of the group) while the second term $\sum_i \frac{\pi_{\Theta_a^i}}{\pi_{\Theta_a}} H(\rho_a^i)$ favors queries that transfer an entire group of objects to one of its child nodes. 

\subsection{Connection to Impurity-based Decision Tree Induction}
\label{sec:connection to impurity-based DT induction}

As a brief digression, in this section we show a connection between the above algorithm and impurity-based decision tree induction. In particular, we show that the above algorithm is equivalent to the decision tree splitting algorithm used in the C$4.5$ software package \cite{quinlan}. Before establishing this result, we briefly review the multi-class classification setting where impurity-based decision tree induction is popularly used. 

In the multi-class classification setting, the input is training data $\vx_1,\cdots,\vx_M$ sampled from some input space (with an underlying probability distribution) along with their class labels, $y_1,\cdots,y_M$ and the task is to construct a classifier with the least probability of misclassification. Decision tree classifiers are grown by maximizing an impurity-based objective function at every internal node to select the best classifier from a set of base classifiers. These base classifiers can vary from simple axis-orthogonal splits to more complex non-linear classifiers. The impurity-based objective function is
\begin{eqnarray}
\label{eq:impurity based objective}
I(\Theta_a ) - \left [ \frac{\pi_{\Theta_{l(a)}}}{\pi_{\Theta_a}} I(\Theta_{l(a)} ) + \frac{\pi_{\Theta_{r(a)}}}{\pi_{\Theta_a}} I(\Theta_{r(a)} ) \right ],
\end{eqnarray} 
which represents the decrease in impurity resulting from split `$a$'. Here $I(\Theta_a)$ corresponds to the measure of impurity in the input subspace at node `$a$' and $\pi_{\Theta_a}$ corresponds to the probability measure of the input subspace at node `$a$'.

Among the various impurity functions suggested in literature \cite{kearns,takimoto}, the entropy measure used in the C$4.5$ software package \cite{quinlan} is popular. In the multi-class classification setting with $m$ different class labels, this measure is given by
\begin{eqnarray}
\label{eq:entropy measure}
& I(\Theta_a) = - \sum_{i=1}^m \frac{\pi_{\Theta_a^i}}{\pi_{\Theta_a}} \log \frac{\pi_{\Theta_a^i}}{\pi_{\Theta_a}} &
\end{eqnarray}
where $\pi_{\Theta_a},\pi_{\Theta_a^i}$ are empirical probabilities based on the training data.

Similar to a query learning problem for group identification, the input here is a binary matrix $\B$ with $b_{ij}$ denoting the binary label produced by base classifier $j$ on training sample $i$, and a probability distribution $\cP$ on the training data along with their class labels $\vy$. But unlike in a query learning problem where the nodes in a tree are not terminated until all the objects belong to the same group, the leaf nodes here are allowed to contain some impurity in order to avoid overfitting. The following result extends Theorem \ref{thm:group identification} to the case of impure leaf nodes.

\begin{thm}
\label{thm:impurity based}
The expected depth of a leaf node in a decision tree classifier $\T$ with reduction factors $\{\rho_a \}_{a \in \cI}$ and class reduction factors $\{\rho_a^i\}_{i=1}^m, \forall a \in \cI$ constructed on a multi-class classification problem $(\B,\cP,\vy)$, is given by
\begin{equation}
\label{eq:relation impurity based}
\E[K] = H(\cP_{\vy} ) + \sum_{a \in \cI} \pi_{\Theta_{a}} \left [1 - H(\rho_a) + \sum_{i=1}^m \frac{\pi_{\Theta_a^i}}{\pi_{\Theta_a}} H(\rho_a^i) \right ] - \sum_{a \in \cL} \pi_{\Theta_a} I(\Theta_a) 
\end{equation}
where $\cP_{\vy}$ denotes the probability distribution of the classes induced by the class labels $\vy$, i.e., $\cP_{\vy} = (\pi_{\Theta^1},\cdots,\pi_{\Theta^m})$ and $I(\Theta_a)$ denotes the impurity in leaf node `$a$' given by (\ref{eq:entropy measure}).
\end{thm}
\begin{proof}
The proof is given in Appendix $\mathrm{I}$.
\end{proof}

The only difference compared to Theorem \ref{thm:group identification} is the last term, which corresponds to the average impurity in the leaf nodes.  

\begin{thm}
\label{thm:connection to impurity based DT induction}
At every internal node in a tree, minimizing the objective function $C_a := 1 - H(\rho_a) + \sum_{i=1}^m \frac{\pi_{\Theta_a^i}}{\pi_{\Theta_a}} H(\rho_a^i)$ is equivalent to maximizing $I(\Theta_a ) - \left [ \frac{\pi_{\Theta_{l(a)}}}{\pi_{\Theta_a}} I(\Theta_{l(a)} ) + \frac{\pi_{\Theta_{r(a)}}}{\pi_{\Theta_a}} I(\Theta_{r(a)} ) \right ]$ with entropy measure as the impurity function.
\end{thm}
\begin{proof}
The proof is given in Appendix $\mathrm{I}$.
\end{proof}

Therefore, greedy optimization of (\ref{eq:relation impurity based}) at internal nodes corresponds to greedy optimization of impurity. Also, note that optimizing (\ref{eq:relation impurity based}) at a leaf assigns the majority vote class label. Therefore, we conclude that impurity-based decision tree induction with entropy as the impurity measure amounts to a greedy optimization of the expected depth of a leaf node in the tree. Also, Theorem \ref{thm:impurity based} allows us to interpret impurity based splitting algorithms for multiclass decision trees in terms of reduction factors, which also appears to be a new insight.

\section{Object identification under group queries}
\label{sec:group queries}
 
In this section, we return to the problem of object identification. The input is a binary matrix $\B$ denoting the relationship between $M$ objects and $N$ queries, where the queries are grouped \emph{a priori} into $n$ disjoint categories, along with the \emph{a priori} probability distribution $\cP$ on the objects. However, unlike the decision trees constructed in the previous two sections where the end user (for e.g., a first responder) has to go through a fixed set of questions as dictated by the decision tree, here, the user is offered more flexibility in choosing the questions at each stage. More specifically, the decision tree suggests a query group from the $n$ groups instead of a single query at each stage, and the user can choose a query to answer from the suggested query group.   

A decision tree constructed with a group of queries at each stage has multiple branches at each internal node, corresponding to the size of the query group. Hence, a tree constructed in this fashion has multiple leaves ending in the same object. While traversing this decision tree, the user chooses the path at each internal node by selecting the query to answer from the given list of queries. Figure \ref{fig:decision tree group queries} demonstrates a decision tree constructed in this fashion for the toy example shown in Figure \ref{fig:toy example 2}. The circled nodes correspond to the internal nodes, where each internal node is associated with a query group. The numbers associated with a dashed edge correspond to the probability that the user will choose that path over the others. The probability of reaching a node $a \in \cI$ in the tree given $\theta \in \Theta_a$ is given by the product of the probabilities on the dashed edges along the path from the root node to that node, for example, the probability of reaching leaf node $\theta_1^*$ given $\theta = \theta_1$ in Figure \ref{fig:decision tree group queries} is $0.45$. The problem now is to select the query categories that will identify the object most efficiently, on average.

\begin{figure*}[!t]
\renewcommand{\arraystretch}{1.3}
\begin{minipage}[b]{0.25\linewidth}
\centering
\begin{tabular}{|c|| c c | c c|}
\hline
& \multicolumn{2}{c|}{$Q^1$} & \multicolumn{2}{c|}{$Q^2$} \\
\cline{2-5}
& 0.5 & 0.5 & 0.9 & 0.1 \\
\cline{2-5}
& $q_1$ & $q_2$ & $q_3$ & $q_4$  \\
\hline \hline
$\theta_1$ & 0 & 1 & 1 & 0 \\
$\theta_2$ & 1 & 0 & 1 & 1 \\
$\theta_3$ & 1 & 1 & 0 & 1 \\
\hline
\end{tabular}
\caption{\small \sl Toy Example 2}
\label{fig:toy example 2}
\end{minipage}
\hspace{0.5cm}
\begin{minipage}[b]{0.75\linewidth}
\centering
\begin{displaymath}
\xymatrix{
 & & &  *++[o][F-]{Q^2} \ar@{-->}[dll]_{0.9} \ar@{-->}[drr]^{0.1} & & & & \\
 & *+[F]{q_3} \ar[dl]_0 \ar[dr]^1 & & & & *+[F]{q_4} \ar[dl]_0 \ar[dr]^1 &\\
 \theta_3 &  & *++[o][F-]{Q^1} \ar@{-->}[dl]_{0.5} \ar@{-->}[dr]^{0.5} & & \theta_1 & & *++[o][F-]{Q^2} \ar@{-->}[d]_1 & \\
  & *+[F]{q_1} \ar[dl]_0 \ar[d]^1 & & *+[F]{q_2} \ar[d]_0 \ar[dr]^1 & & & *+[F]{q_3} \ar[dl]_0 \ar[dr]^1 & \\
 \theta_1^* & \theta_2 & & \theta_2 & \theta_1 & \theta_3 & & \theta_2 \\
}
\end{displaymath}
\caption{\small \sl Decision tree constructed on toy example 2 for object identification under group queries}
\label{fig:decision tree group queries}
\end{minipage}
\end{figure*}

In addition to the terminology defined in Sections \ref{sec:notation} and \ref{sec:object identification}, we also define $\vz = (z_1,\cdots,z_N)$ to be the group labels of the queries, where $z_j \in \{1,\cdots,n\}, \forall j = 1,\cdots, N$. Let $\{Q^i\}_{i=1}^n$ be a partition of the query set $Q$, where $Q^i$ denotes the set of queries in $Q$ that belong to group $i$. Similarly, at any node `$a$' in a tree, let $Q_a^i$ and $\overline{Q_a^i}$ denote the set of queries in $Q_a$ and $Q \setminus Q_a$ that belong to group $i$ respectively. Let $p_i(q)$ be the \emph{a priori} probability of the user selecting query $q \in Q^i$ at any node with query group $i$ in the tree, where $\sum_{q \in Q^i} p_i(q) = 1$. In addition, at any node `$a$' in the tree, the function $p_i(q) = 0, \forall q \in Q_a^i$, since the user would not choose a query which has already been answered, in which case $p_i(q)$ is renormalized. In our experiments we take $p_i(q)$ to be uniform on $\overline{Q_a^i}$. Finally, let $z_a \in \{1,\cdots,n\}$ denote the query group selected at an internal node `$a$' in the tree and let $\p_a$ denote the probability of reaching that node given $\theta \in \Theta_a$.

We denote a query learning problem for object identification with query groups by $(\B,\cP,\vz,\vp)$. Given $(\B,\cP,\vz,\vp)$, let $\cT(\B,\cP,\vz,\vp)$ denote the set of decision trees that can uniquely identify all the objects in the set $\Theta$ with query groups at each internal node. For a decision tree $\T \in \cT(\B,\cP,\vz,\vp)$, let $\{\rho_{a}(q)\}_{q \in Q^{z_a}}$ denote the reduction factors of all the queries in the query group at each internal node $a \in \cI$ in the tree, where the reduction factors are treated as functions with input being a query. 

Also, for a tree with $L$ leaves, let $\cL^i \subset \cL = \{1,\cdots,L\}$ denote the set of leaves terminating in object $\theta_i$ and let $d_j$ denote the depth of leaf node $j \in \cL$. Then, the expected number of queries required to identify the unknown object using the given tree is equal to
\begin{eqnarray*}
\E[K] &=& \sum_{i=1}^M \prob(\theta = \theta_i) \E[K|\theta = \theta_i] \\
& = & \sum_{i=1}^M \pi_i \left [ \sum_{j \in \cL^i} \p_j d_j \right ]
\end{eqnarray*} 

\begin{thm}
\label{thm:group queries}
The expected number of queries required to identify an object using a tree $\T \in \cT(\B,\cP,\vz,\vp)$ is given by
\begin{equation}
\label{eq:group queries}
\E[K] = H(\cP) + \sum_{a \in \cI} \p_a \pi_{\Theta_{a}} \left [1 - \sum_{q \in Q^{z_a}} p_{z_a}(q) H(\rho_{a}(q) ) \right ] 
\end{equation}
\end{thm}
\begin{proof}
Special case of Theorem \ref{thm:group identification group queries} below.
\end{proof}

Note from the above theorem, that given a query learning problem $(\B,\cP,\vz,\vp)$, the expected number of queries required to identify an object is lower bounded by its entropy $H(\cP)$. Also, this lower bound can be achieved iff the reduction factors of all the queries in a query group at each internal node of the tree is equal to $0.5$. In fact, Theorem \ref{thm:object identification} is a special case of the above theorem where each query group has just one query. 

Given a query learning problem $(\B,\cP,\vz,\vp)$, the problem of finding a decision tree with minimum $\E[K]$ can be formulated as the following optimization problem
\begin{eqnarray}
\label{eq:optimization group queries}
& \underset{\T \in \cT(\B,\cP,\vz,\vp)}{\operatorname{\min}}  \sum_{a \in \cI} \p_a \pi_{\Theta_{a}} \left [1 - \sum_{q \in Q^{z_a}} p_{z_a}(q) H(\rho_{a}(q) ) \right ] &
\end{eqnarray}

Note that here the reduction factors $\rho_{a}(q), \forall q \in Q^{z_a}$ and the prior probability function $p_{z_a}(q)$ depends on the query group $z_a \in \{1,\cdots,n\}$ chosen at node `$a$' in the tree. The above optimization problem being a generalized version of the optimization problem in (\ref{eq:optimization object identification}) is NP-complete. A greedy top-down local optimization of the above objective function yields a suboptimal solution where we choose a query group that minimizes the term $C_a(j) := \left [1 - \sum_{q \in Q^{j}} p_{j}(q) H(\rho_{a}(q) ) \right ]$ at each internal node, starting from the root node. The algorithm as summarized in Algorithm \ref{algo_groupqueries} below is referred to as GQSA (Group Queries Splitting Algorithm) in the rest of this paper.

\restylealgo{boxed}
\begin{algorithm}
\dontprintsemicolon
\caption{Greedy decision tree algorithm for object identification with group queries\label{algo_groupqueries}}

\textbf{\underline{Group Queries Splitting Algorithm (GQSA)}} \;
\BlankLine
\textbf{Initialization :}  \emph{Let the leaf set consist of the root node} \;
\SetLine
\While{some leaf node `$a$' has $|\Theta_a| > 1$}{
\For{each query group with $\left |\overline{Q_a^j} \right | \geq 1$}{
Compute the prior probabilities of selecting queries within a group $p_j(q), \forall q \in Q^j$ at node `$a$'\;
Compute the reduction factors for all the queries in the query group $\{\rho_{a}(q)\}_{q \in Q^j}$\; 
Compute the cost $C_a(j)$ of using query group $j$ at node `$a$'\;
}
Choose a query group $j$ with the least cost $C_a(j)$ at node `$a$'\;
Form the left and the right child nodes for all queries with $p_j(q) > 0$ in the query group\;
}
\end{algorithm}

\section{Group identification under group queries}
\label{sec:group identification group queries}

For the sake of completion, we consider here the problem of identifying the group of an unknown object $\theta \in \Theta$ under group queries. The input is a binary matrix $\B$ denoting the relationship between $M$ objects and $N$ queries, where the objects are grouped into $m$ groups and the queries are grouped into $n$ groups. The task is to identify the group of an unknown object through as few queries from $Q$ as possible where, at each stage, the user is offered a query group from which a query is chosen.

As noted in Section \ref{sec:group identification}, a decision tree constructed for group identification can have multiple objects terminating in the same leaf node. Also, a decision tree constructed for group identification with a query group at each internal node has multiple leaves terminating in the same group. Hence a decision tree constructed in this section can have multiple objects terminating in the same leaf node and multiple leaves terminating in the same group. Also, we use most of the terminology defined in Sections \ref{sec:group identification} and \ref{sec:group queries} here.

We denote a query learning problem for group identification with query groups by $(\B,\cP,\vy,\vz,\vp)$ where $\vy = (y_1,\cdots,y_M)$ denotes the group labels on the objects, $\vz = (z_1,\cdots,z_N)$ denotes the group labels on the queries and $\vp = (p_1(q),\cdots,p_n(q))$ denotes the \emph{a priori} probability functions of selecting queries within query groups. Given a query learning problem $(\B,\cP,\vy,\vz,\vp)$, let $\cT(\B,\cP,\vy,\vz,\vp)$ denote the set of decision trees that can uniquely identify the groups of all objects in the set $\Theta$ with query groups at each internal node. For any decision tree $\T \in \cT(\B,\cP,\vy,\vz,\vp)$, let $\{\rho_{a}(q)\}_{q \in Q^{z_a}}$ denote the reduction factor set and let $\{\{\rho_{a}^i(q)\}_{i=1}^m \}_{q \in Q^{z_a}}$ denote the group reduction factor sets at each internal node $a \in \cI$ in the tree, where $z_a \in \{1,\cdots,n\}$ denotes the query group selected at that node. 

Also, for a tree with $L$ leaves, let $\cL^i \subset \cL = \{1,\cdots,L\}$ denote the set of leaves terminating in object group $i$ and let $d_j,\p_j$ denote the depth of leaf node $j \in \cL$ and the probability of reaching that node given $\theta \in \Theta_j$, respectively. Then, the expected number of queries required to identify the group of an unknown object using the given tree is equal to
\begin{eqnarray*}
\E[K] &=& \sum_{i=1}^m \prob(\theta \in \Theta^i) \E[K|\theta \in \Theta^i] \\
& = & \sum_{i=1}^m \pi_{\Theta^i} \left [ \sum_{j \in \cL^i} \frac{\pi_{\Theta_j}}{\pi_{\Theta^i}} \p_j d_j \right ]
\end{eqnarray*} 

\begin{thm}
\label{thm:group identification group queries}
The expected number of queries required to identify the group of an unknown object using a tree $\T \in \cT(\B,\cP,\vy,\vz,\vp)$ is given by
\begin{equation}
\label{eq:group identification group queries}
\E[K] = H(\cP_{\vy}) + \sum_{a \in \cI} \p_a \pi_{\Theta_{a}} \left \{ 1 - \sum_{q \in Q^{z_a} } p_{z_a}(q)\left [ H(\rho_{a}(q) ) - \sum_{i=1}^m \frac{\pi_{\Theta_{a}^i}}{\pi_{\Theta_a}} H(\rho_{a}^i(q)) \right ] \right \} 
\end{equation}
where $\cP_{\vy}$ denotes the probability distribution of the object groups induced by the labels $\vy$, i.e. $\cP_{\vy} = (\pi_{\Theta^1},\cdots,\pi_{\Theta^m})$
\end{thm}
\begin{proof}
The proof is given in Appendix $\mathrm{I}$.
\end{proof}

Note that Theorems \ref{thm:object identification}, \ref{thm:group identification} and \ref{thm:group queries} are special cases of the above theorem. This theorem states that, given a query learning problem $(\B,\cP,\vy,\vz,\vp)$, the expected number of queries required to identify the group of an object is lower bounded by the entropy of the probability distribution of the object groups $H(\cP_{\vy})$. It also follows from the above theorem that this lower bound can be achieved iff the reduction factors and the group reduction factors of all the queries in a query group at each internal node are equal to $0.5$ and $1$ respectively. 

The problem of finding a decision tree with minimum $\E[K]$ can be formulated as the following optimization problem

\begin{eqnarray}
\label{eq:optimization group identification group queries}
& \underset{\T \in \cT(\B,\cP,\vy,\vz,\vp)}{\operatorname{\min}}  \sum_{a \in \cI} \p_a \pi_{\Theta_{a}} \left \{ 1 - \sum_{q \in Q^{z_a} } p_{z_a}(q)\left [ H(\rho_{a}(q) ) - \sum_{i=1}^m \frac{\pi_{\Theta_{a}^i}}{\pi_{\Theta_a}} H(\rho_{a}^i(q)) \right ] \right \}   &
\end{eqnarray}

Note that here the reduction factors $\{\rho_{a}(q)\}_{q \in Q^{z_a}}$, the group reduction factors $\{\rho^i_{a}(q)\}_{q \in Q^{z_a}}$ for all $i = 1,\cdots,m$, and the prior probability function $p_{z_a}(q)$ depends on the query group $z_a \in \{1,\cdots,n\}$ chosen at node `$a$' in the tree. Once again, the above optimization problem being a generalized version of the optimization problem in (\ref{eq:optimization object identification}) is NP-complete. A greedy top-down optimization of the above objective function yields a suboptimal solution where we choose a query group that minimizes the term $C_a(j) := 1 - \sum_{q \in Q^{j}} p_{j}(q)\left [ H(\rho_{a}(q) ) - \sum_{i=1}^m \frac{\pi_{\Theta_{a}^i}}{\pi_{\Theta_a}} H(\rho_{a}^i(q)) \right ] $ at each internal node, starting from the root node. The algorithm as summarized in Algorithm \ref{algo_groupidentificationgroupqueries} below is referred to as GIGQSA (Group Identification under Group Queries Splitting Algorithm).

\begin{algorithm}
\dontprintsemicolon
\caption{Greedy decision tree algorithm for group identification under group queries\label{algo_groupidentificationgroupqueries}}

\textbf{\underline{Group Identification under Group Queries Splitting Algorithm (GIGQSA)}} 
\BlankLine
\textbf{Initialization :}  \emph{Let the leaf set consist of the root node} \;
\SetLine
\While{some leaf node `$a$' has more than one group of objects}{
\For{each query group with $\left |\overline{Q_a^j} \right | \geq 1$}{
Compute the prior probabilities of selecting queries within a group, $p_{j}(q), \forall q \in Q^j$ at node `$a$'\;
Compute the reduction factors for all the queries in the query group $\{\rho_{a}(q)\}_{q \in Q^j}$\; 
Compute the group reduction factors for all the queries in the query group $\{ \rho_{a}^i(q) \}_{ q \in Q^j }$, $\forall i = 1,\cdots,m$\;
Compute the cost $C_a(j)$ of using query group $j$ at node `$a$'\;
}
Choose a query group $j$ with the least cost $C_a(j)$ at node `$a$'\;
Form the left and the right child nodes for all queries with $p_j(q) > 0$ in the query group\;
}
\end{algorithm}

\section{Query learning with persistent noise}
\label{sec:QL with persistent noise}
We now consider the problem of identifying an unknown object $\theta \in \Theta$ through as few queries as possible in the presence of persistent query noise, and relate this problem to group identification. Query noise refers to errors in the query responses, i.e., the observed query response is different from the true response of the unknown object. For example, a victim of toxic chemical exposure may not report a symptom because of a delayed onset of that symptom. Unlike the noise model often assumed in the literature, where repeated querying results in independent realizations of the noise, persistent query noise is a more stringent noise model where repeated queries result in the same response. 

We refer to the bit string consisting of observed query responses as an input string. The input string can differ from the true bit string (corresponding to the row vector of the true object in matrix $\B$) due to persistent query noise. First, we describe the error model and then describe the application of group identification algorithms to uniquely identify the true object in the presence of persistent errors.

Consider the case where a fraction $\nu$ of the $N$ queries are prone to error. Also, assume that at any instance, not more than $\epsilon$ of these $\nu N$ queries are in error, where $\epsilon := \lfloor \frac{\delta - 1}{2} \rfloor$, $\delta$ being the minimum Hamming distance between any two rows of the matrix $\B$. The \emph{a priori} probability distribution of the number of errors is considered to be one of the following, 
\begin{eqnarray*}
\mbox{Probability model 1:} \ \ \ & & \prob (e \ \ \mbox{errors}) = \frac{{N\nu \choose e}}{\sum_{e'=0}^{\epsilon'} {N\nu \choose e'}}, \ \ \ 0 \leq e \leq \epsilon' \\  
\mbox{Probability model 2:} \ \ \ & & \prob (e \ \ \mbox{errors}) = \frac{{N\nu \choose e} p^e (1-p)^{N\nu - e}}{\sum_{e'=0}^{\epsilon'} {N\nu \choose e'}p^{e'} (1-p)^{N\nu - e'}}, \ \ \ 0 \leq e \leq \epsilon'
\end{eqnarray*}
where $\epsilon' := \min (\epsilon,N\nu)$. Note that probability model $2$ corresponds to a truncated binomial distribution where $0 \leq p \leq 0.5$ denotes the probability that a query prone to error is actually in error, while probability model $1$ is a special case of probability model $2$ when $p = 0.5$. Given this error model, the goal is to identify the true object through as few queries from $Q$ as possible.

This problem can be posed as a group identification problem as follows: Given a query learning problem $(\B,\cP)$ with $M$ objects and $N$ queries that is susceptible to $\epsilon$ errors, with a fraction $\nu$ of the $N$ queries prone to error, create $(\Bt,\cPt)$ with $M$ groups of objects and $N$ queries, where each object group in this new matrix consists of $\sum_{e = 0}^{\epsilon'} {N\nu \choose e}$ objects corresponding to all possible bit strings that differ from the original bit string in at most $\epsilon'$ positions corresponding to the $\nu N$ bits prone to error. Consider the toy example shown in Figure \ref{fig:toy example3} consisting of $2$ objects and $3$ queries with an $\epsilon = 1$ where queries $q_2$ and $q_3$ are prone to persistent query noise. Figure \ref{fig:dilated matrix example} demonstrates the construction of $\Bt$ for this toy example.

\begin{figure*}[!t]
\renewcommand{\arraystretch}{1.3}
\centerline{\subfigure[]{
\begin{tabular}{|c|c c c|c|}
\hline
& $q_1$ & $q_2$ & $q_3$ & $\cP$ \\
prone to error & $\times $ & $\checkmark$ & $\checkmark$ & \\
\hline
$\theta_1$ & 0 & 0 & 0 & $\frac{1}{4}$ \\
$\theta_2$ & 1 & 1 & 1 & $\frac{3}{4}$ \\
\hline
\end{tabular}
\label{fig:toy example3}}
\hfil
\subfigure[]{
\begin{tabular}{|c|c c c|c|c|}
\hline
& $q_1$ & $q_2$ & $q_3$ & $\cPt_1 (p = 0.5)$ & $\cPt_2 (p = 0.25)$  \\
\hline
 & 0 & 0 & 0 & $\frac{1}{12}$ & $\frac{3}{20}$  \\
 $\Theta^1$ & 0 & 1 & 0 & $\frac{1}{12}$ & $\frac{1}{20}$\\
 & 0 & 0 & 1 & $\frac{1}{12}$ & $\frac{1}{20}$\\
 \hline
 & 1 & 1 & 1 & $\frac{1}{4}$ & $\frac{9}{20}$ \\
 $\Theta^2$ & 1 & 0 & 1 & $\frac{1}{4}$ & $\frac{3}{20}$\\
 & 1 & 1 & 0 & $\frac{1}{4}$ & $\frac{3}{20}$\\
\hline
\end{tabular}
\label{fig:dilated matrix example}}}
\caption{\small \sl For the toy example shown in (a) consisting of $2$ objects and $3$ queries with an $\epsilon = 1$, where queries $q_2$ and $q_3$ are prone to persistent noise, (b) demonstrates the construction of matrix $\Bt$}
\end{figure*}

Each bit string in the object set $\Theta^i$ corresponds to one of the possible input strings when the true object is $\theta_i$ and at most $\epsilon'$ errors occur. Also, by definition of $\epsilon$, no two bit strings in the matrix $\Bt$ are the same. Given the \emph{a priori} probabilities of the objects in $\B$, the prior distribution of objects in $\Bt$ is generated as follows. For an object belonging to group $i$ in $\Bt$ whose bit string differs in $e \leq \epsilon'$ bit positions from the true bit string of $\theta_i$, the prior probability is given by
\begin{eqnarray*}
\mbox{Probability model 1:} \ \ \ & & \frac{1}{\sum_{e' = 0}^{\epsilon'} {N\nu \choose e'} } \ \pi_i \\
\mbox{Probability model 2:} \ \ \ & & \frac{p^e (1-p)^{N\nu - e}}{\sum_{e' = 0}^{\epsilon'} {N\nu \choose e'} p^{e'} (1-p)^{N\nu - e'}} \ \pi_i
\end{eqnarray*}
Figure \ref{fig:dilated matrix example} shows the prior probability distribution of the objects in $\Bt$ using probability model $1$ $(\cPt_1)$ and probability model $2$ with $p = 0.25$ $(\cPt_2)$ for the toy example shown in Figure \ref{fig:toy example3}.

Given a query learning problem $(\B,\cP)$ that is susceptible to $\epsilon$ errors, the problem of identifying an unknown object in the presence of at most $\epsilon$ persistent errors can be reduced to the problem of identifying the group of an unknown object in $(\Bt,\cPt)$, where $(\Bt,\cPt)$ is generated as described above. One possible concern with this approach could be any memory related issue in generating matrix $\Bt$ due to the combinatorial explosion in the number of objects in $\Bt$. Interestingly, the relevant quantities for query selection in both GISA and GBS (i.e., the reduction factors) can be efficiently computed without explicitly constructing the $\Bt$ matrix, described in detail in Appendix $\mathrm{II}$.

\section{Experiments}
We perform three sets of experiments, demonstrating our algorithms for group identification, object identification using query groups, and query learning with persistent noise. In each case, we compare the performances of the proposed algorithms to standard algorithms such as the splitting algorithm, using synthetic data as well as a real dataset, the WISER database. The WISER database is a toxic chemical database describing the binary relationship between $298$ toxic chemicals and $79$ acute symptoms. The symptoms are grouped into $10$ categories (e.g., neurological, cardio) as determined by NLM, and the chemicals are grouped into $16$ categories (e.g., pesticides, corrosive acids) as determined by a toxicologist and a Hazmat expert.  

\subsection{Group identification}
Here, we consider a query learning problem $(\B,\cP)$ where the objects are grouped into $m$ groups given by $\vy = (y_1,\cdots,y_M)$, $y_i \in \{1,\cdots,m\}$, with the task of identifying the group of an unknown object from the object set $\Theta$ through as few queries from $Q$ as possible. First, we consider random datasets generated using a random data model and compare the performances of GBS and GISA for group identification in these random datasets. Then, we compare the performance of the two algorithms in the WISER database. In both these experiments, we assume a uniform \emph{a priori} probability distribution on the objects.
\subsubsection{Random Datasets}
\label{sec:random networks group identification}
We consider random datasets of the same size as the WISER database, with $298$ objects and $79$ queries where the objects are grouped into $16$ classes with the same group sizes as that in the WISER database. We associate each query in a random dataset with two parameters, $\gamma_w \in [0.5,1]$ which reflects the correlation of the object responses \textit{within} a group, and $\gamma_b \in [0.5,1]$ which captures the correlation of the object responses \textit{between} groups. When $\gamma_w$ is close to $0.5$, each object within a group is equally likely to exhibit $0$ or $1$ as its response to the query, whereas, when $\gamma_w$ is close to $1$, most of the objects within a group are highly likely to exhibit the same response to the query. Similarly, when $\gamma_b$ is close to $0.5$, each group is equally likely to exhibit $0$ or $1$ as its response to the query, where a group response corresponds to the majority vote of the object responses within a group, while, as $\gamma_b$ tends to $1$, most of the groups are highly likely to exhibit the same response.    

Given a $(\gamma_w,\gamma_b)$ pair for a query in a random dataset, the object responses for that query are created as follows
\begin{enumerate}
\item
Generate a Bernoulli random variable, $x$
\item
For each group $i \in \{1,\cdots,m\}$, assign a binary label $b_i$, where $b_i = x$ with probability $\gamma_b$
\item
For each object in group $i$, assign $b_i$ as the object response with probability $\gamma_w$ 
\end{enumerate} 
Given the correlation parameters $(\gamma_w(q),\gamma_b(q)) \in [0.5,1]^2, \forall q \in Q$, a random dataset can be created by following the above procedure for each query. Conversely, we describe in Section \ref{sec:WISER group identification} on how to estimate these parameters for a given dataset. 

Figure \ref{fig:random experiment 1} compares the mean $\E[K]$ for GBS and GISA in $100$ randomly generated datasets (for each value of $d_1$ and $d_2$), where the random datasets are created such that the query parameters are uniformly distributed in the rectangular space governed by $d_1,d_2$ as shown in Figure \ref{fig:gamma parameter space}. This demonstrates the improved performance of GISA over GBS in group identification. Especially, note that $\E[K]$ tends close to entropy $H(\cP_{\vy})$ using GISA as $d_2$ increases.  

\begin{figure}[!t]
  \centering
       \includegraphics[width = 0.8\textwidth]{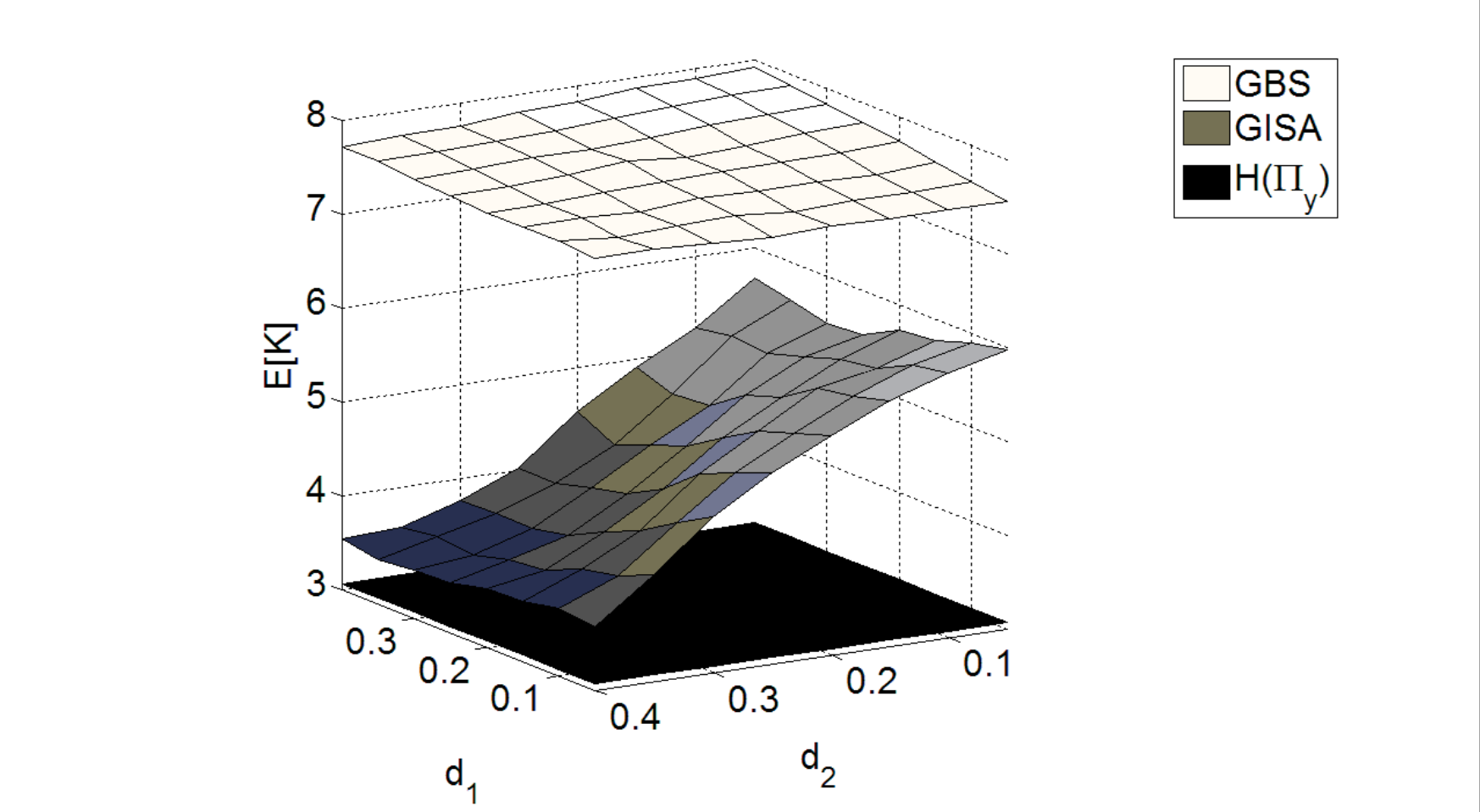}
 \caption{\small \sl Expected number of queries required to identify the group of an object using GBS and GISA on random datasets generated using the proposed random data model}
 \label{fig:random experiment 1}
\end{figure}  

This is due to the increment in the number of queries in the fourth quadrant of the parameter space as $d_2$ increases. Specifically, as the correlation parameters $\gamma_w,\gamma_b$ tends to $1$ and $0.5$ respectively, choosing that query eliminates approximately half the groups with each group being either completely eliminated or completely included, i.e. the group reduction factors tend to $1$ for these queries. Such queries are preferable in group identification and GISA is specifically designed to search for these queries leading to its strikingly improved performance over GBS as $d_2$ increases.

\begin{figure}[!t]
\begin{minipage}[b]{0.5\linewidth}
\begin{center}
 \includegraphics[width = 1.1\textwidth]{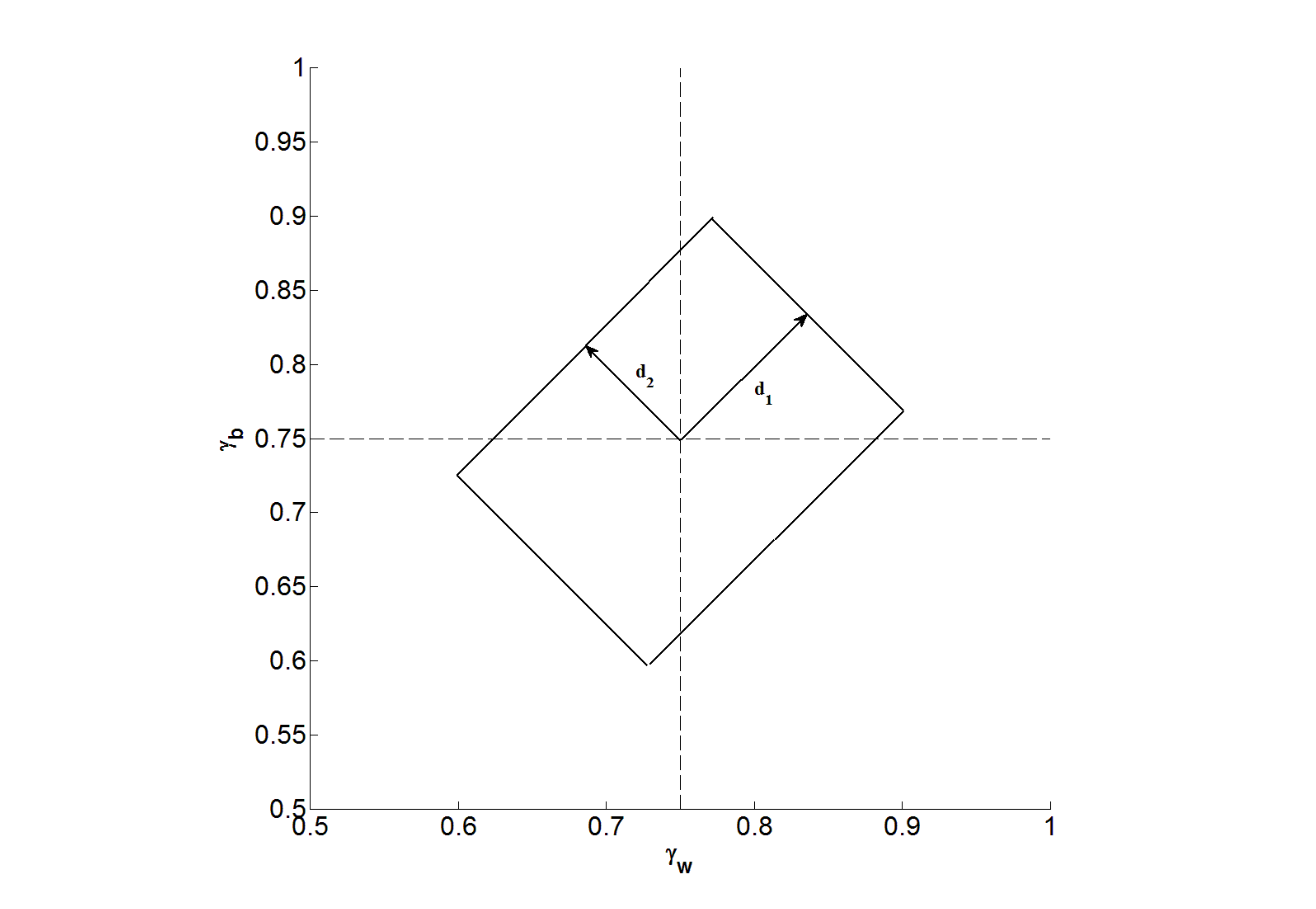}
\end{center}
\caption{\small \sl Random data model - The query parameters $(\gamma_w(q),\gamma_b(q))$ are restricted to lie in the rectangular space}
\label{fig:gamma parameter space}
\end{minipage}
\hspace{0.35cm}
\begin{minipage}[b]{0.5\linewidth}
\begin{center}
  \includegraphics[width = 1.1\textwidth]{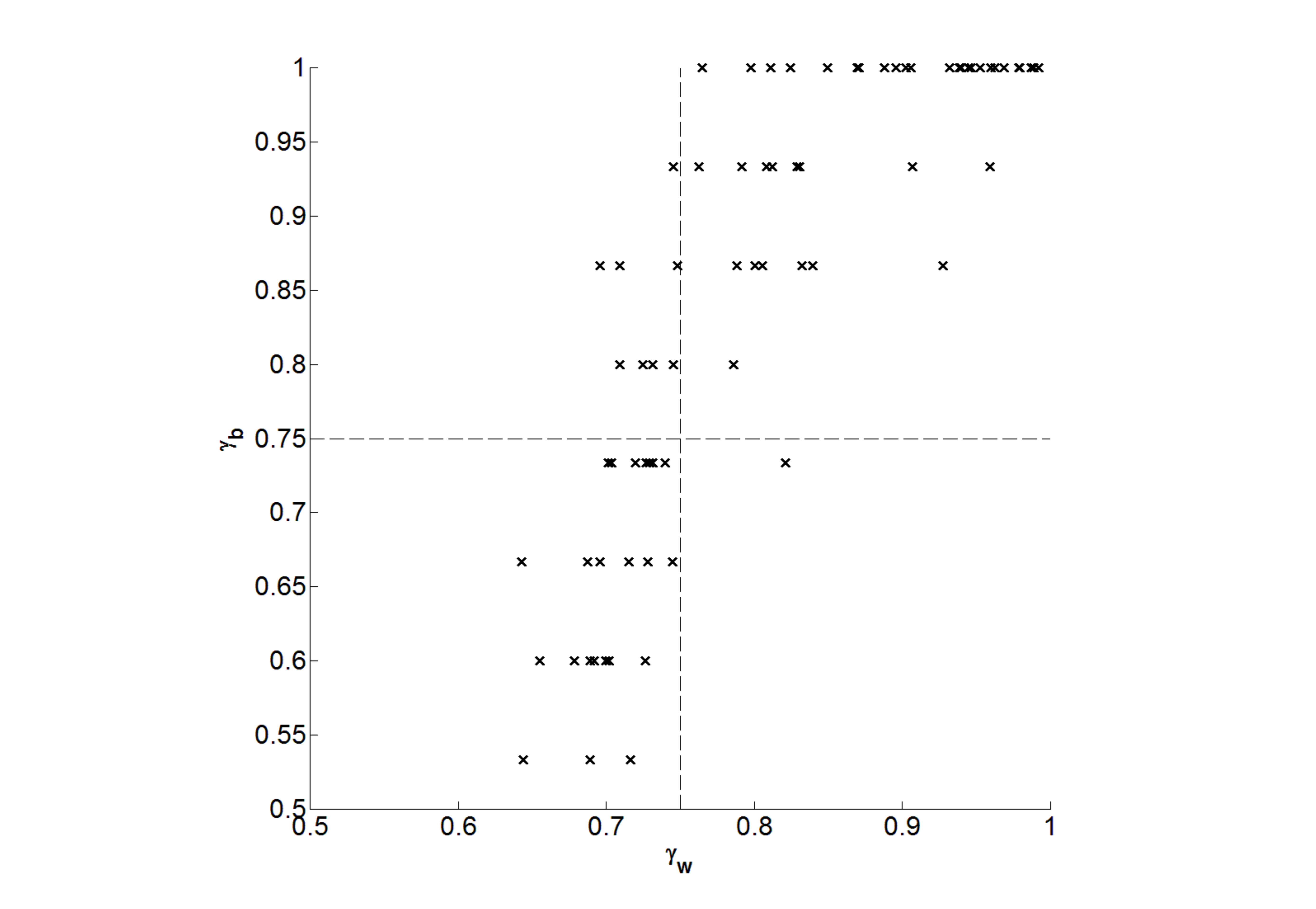}
\end{center}
\caption{\small \sl Scatter plot of the query parameters in the WISER database}
\label{fig:WISER gamma parameters}
\end{minipage}
\end{figure}

\subsubsection{WISER Database}
\label{sec:WISER group identification}

Table \ref{table:WISER group identification} compares the expected number of queries required to identify the group of an unknown object in the WISER database using GISA, GBS and random search, where the group entropy in the WISER database is given by $H(\cP_{\vy}) = 3.068$. The table reports the $95\%$ symmetric confidence intervals based on random trails, where the randomness in GISA and GBS is due to the presence of multiple best splits at each internal node.

However, the improvement of GISA over GBS on WISER is less than was observed for many of the random datasets discussed above. To understand this, we developed a method to estimate the correlation parameters of the queries for a given dataset $\B$. For each query in the dataset, the correlation parameters can be estimated as follows
\begin{enumerate}
\item
For every group $i \in \{1,\cdots,m\}$, let $b_i$ denote the group response given by the majority vote of object responses in the group and let $\widehat{\gamma_w^i}$ denote the fraction of objects in the group with similar response as $b_i$
\item
Denote by a binary variable $x$, the majority vote of the group responses $\vb = [b_1,\cdots,b_m]$
\item
Then, $\widehat{\gamma_b}$ is given by the fraction of groups with similar response as $x$, and $\widehat{\gamma_w} = \frac{1}{m} \sum_i \widehat{\gamma_w^i}$
\end{enumerate}
Now, we use the above procedure to estimate the query parameters for all queries in the WISER database, shown in Figure \ref{fig:WISER gamma parameters}. Note from this figure that there is just one query in the fourth quadrant of the parameter space and there are no queries with $\gamma_w$ close to $1$ and $\gamma_b$ close to $0.5$. In words, chemicals in the same group tend to behave differently and chemicals in different groups tend to exhibit similar response to the symptoms. This is a manifestation of the non-specificity of the symptoms in the WISER database as reported by Bhavnani et. al. \cite{suresh}.

\begin{table}[ht]
\begin{minipage}[b]{0.5 \linewidth}
\label{table:WISER group identification}
\centering
\begin{tabular}{|c|c|}
\hline
Algorithm & $\E[K]$ \\
\hline
GISA & 7.792 $\pm$ 0.001 \\
GBS & 7.948 $\pm$ 0.003 \\
Random Search & 16.328 $\pm$ 0.177 \\
\hline
\end{tabular}
\caption{\small \sl Expected number of queries required to identify the group of an object in WISER database}
\end{minipage}
\hspace{0.5cm}
\begin{minipage}[b]{0.5 \linewidth}
\label{table:WISER group queries}
\centering
\begin{tabular}{|c | c |}
\hline
 Algorithm & $\E[K]$ \\
\hline
GBS & 8.283 $\pm$ 0.000  \\
GQSA & 11.360 $\pm$ 0.096 \\
$\min_i \min_{q \in Q^{i}} p_i(q) \rho_a(q)$ & 13.401 $\pm$ 0.116 \\
$\min_i \max_{q \in Q^i} p_i(q) \rho_a(q)$ & 18.697 $\pm$ 0.357 \\
Random Search & 20.251 $\pm$ 0.318 \\
\hline
\end{tabular}
\caption{\small \sl Expected number of queries required to identify an object under group queries in WISER database}
\end{minipage}
\end{table}

\subsection{Object identification under query classes}
In this section, we consider a query learning problem $(\B,\cP)$ where the queries are \emph{a priori} grouped into $n$ groups given by $\vz = (z_1,\cdots,z_N), z_i \in \{1,\cdots,n\}$, with the task of identifying an unknown object from the set $\Theta$ through as few queries from $Q$ as possible, where the user is presented with a query group at each stage to choose from. Note that this approach is midway between a complete active search strategy and a complete passive search strategy. Hence, we primarily compare the performance of GQSA to a completely active search strategy such as GBS and a completely passive search strategy like random search where the user randomly chooses the queries from the set $Q$ to answer. In addition, we also compare GQSA to other possible heuristics where we choose a query group $i$ that minimizes $\min_{q \in Q^{i}} p_{i}(q) \rho_a(q)$ or $\max_{q \in Q^{i}} p_{i}(q) \rho_a(q)$ at each internal node `$a$'.

First, we compare the performances of these algorithms on random datasets generated using a random data model. Then, we compare them in the WISER database. In both these experiments, we assume uniform \emph{a priori} probability distribution on the objects as well as on queries within a group. The latter probability distribution corresponds to the probability of a user selecting a particular query $q$ from a query group, $p_i(q), \forall i=1,\cdots,n$.

\subsubsection{Random Datasets}
\label{sec:random networks group queries}
Here, we consider random datasets of the same size as the WISER database, with $298$ objects and $79$ queries where the queries are grouped into $10$ groups with the same group sizes as that in the WISER database. We associate a random dataset with a parameter $\gamma_{max} \in [0.5,1]$, where $\gamma_{max}$ corresponds to the maximum permissible value of $\gamma_b$ for a query in the random dataset. Given a $\gamma_{max}$, a random dataset is created as follows
\begin{enumerate}
\item
For each query group, generate a $\gamma_b \in [0.5,\gamma_{max}]$
\item
For each query in the query group, generate a Bernoulli random variable $x$ and give each object the same query label as $x$ with probability $\gamma_b$
\end{enumerate}

Figure \ref{fig:random experiment group queries} compares the mean $\E[K]$ for the respective algorithms in $100$ randomly generated datasets, for each value of $\gamma_{max}$. The $\min \min$ corresponds to the heuristic where we minimize $\min_{q \in Q^{i}} p_{i}(q) \rho_a(q)$ at each internal node and the $\min \max$ corresponds to the heuristic where we minimize $\max_{q \in Q^{i}} p_{i}(q) \rho_a(q)$. Note from the figure that in spite of not being a completely active search strategy, the performance of GQSA is comparable to that of GBS and better than the other algorithms.

\begin{figure}[!t]
\begin{minipage}[b]{0.5\linewidth}
\begin{center}
 \includegraphics[width = 1.1\textwidth]{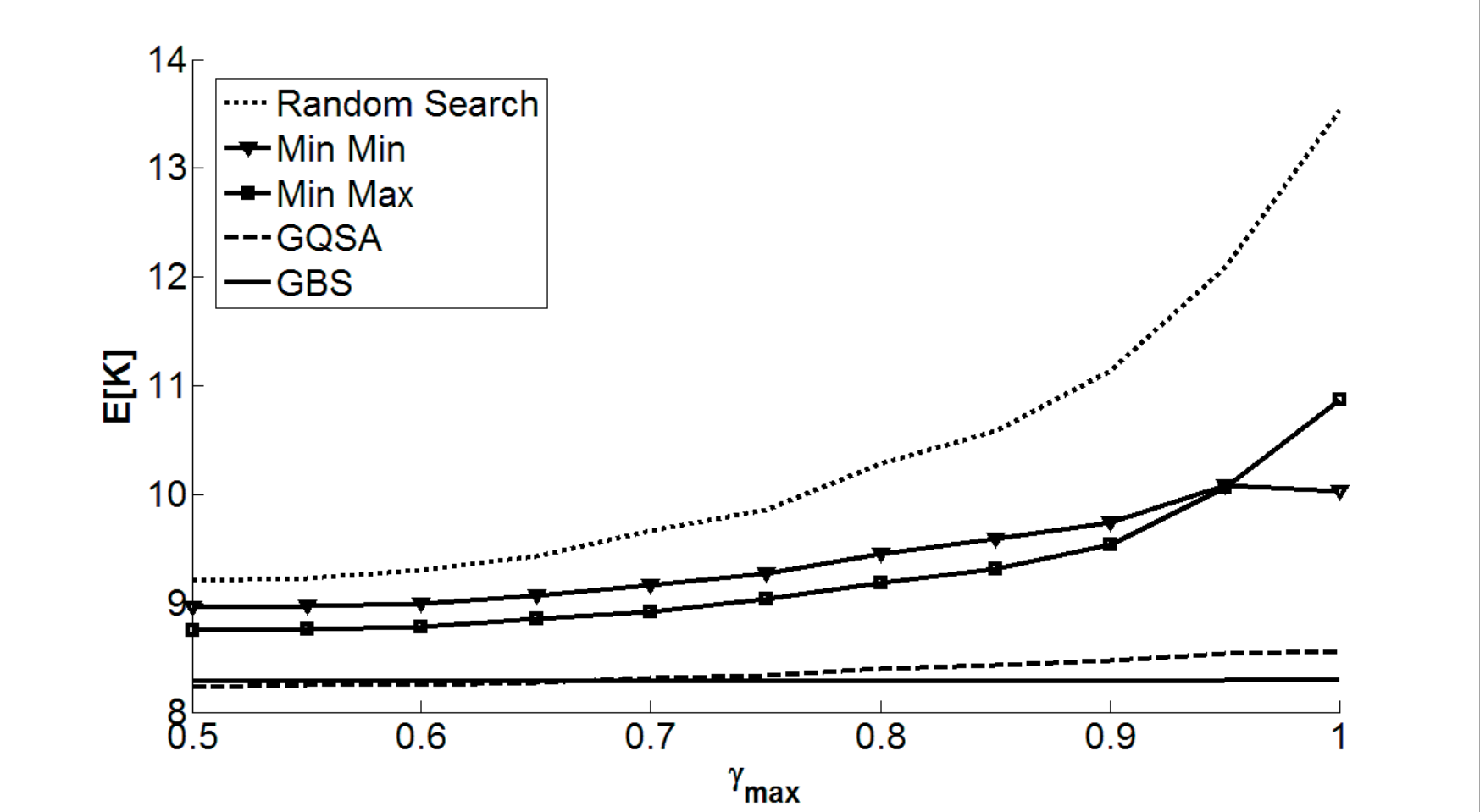}
\end{center}
\caption{\small \sl Expected number of queries required by different algorithms for object identification under group queries in random datasets}
\label{fig:random experiment group queries}
\end{minipage}
\hspace{0.35cm}
\begin{minipage}[b]{0.5\linewidth}
\begin{center}
  \includegraphics[width = 1.1\textwidth]{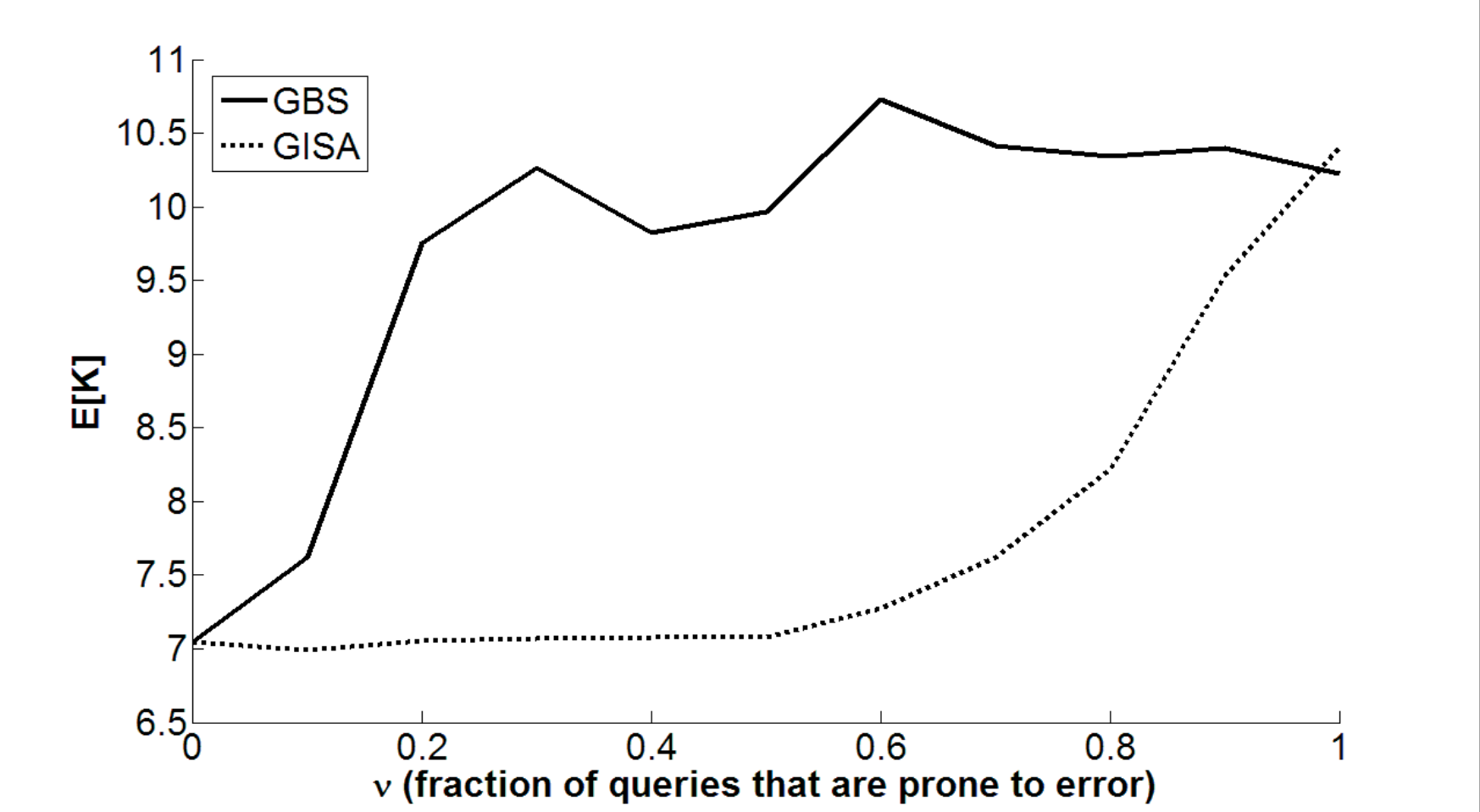}
\end{center}
\caption{\small \sl Comparison between the performance of GBS and GISA in identifying the true object in the presence of restricted persistent noise under probability model $1$}
\label{fig:WISER noise model1}
\end{minipage}
\end{figure}

\subsubsection{WISER Database}
\label{sec:WISER group queries}
Table \ref{table:WISER group queries} compares the expected number of queries required to identify an unknown object under group queries in the WISER database using the respective algorithms, where the entropy of the objects in the WISER database is given by $H(\cP) = 8.219$. The table reports the $95\%$ symmetric confidence intervals based on random trials, where the randomness in GBS is due to the presence of multiple best splits at each internal node. 

Once again, it is not surprising that GBS outperforms GQSA as GBS is fully active, i.e, it always chooses the best split, whereas GQSA does not always pick the best split, since a human is involved. Yet, the performance of GQSA is not much worse than that of GBS. Infact, if we were to fully model the time-delay associated with answering a query, then GQSA might have a smaller ``time to identification,'' because presumably it would take less time to answer the queries on average.

\subsection{Query learning with persistent noise}
In Section \ref{sec:QL with persistent noise}, we showed that identifying an unknown object in the presence of persistent errors can be reduced to a group identification problem. Hence, any group identification algorithm can be adopted to solve this problem. Here, we compare the performance of GBS and GISA in identifying the unknown object in the presence of persistent errors. 

Note from Section \ref{sec:QL with persistent noise} that the generation of matrix $\Bt$ requires the knowledge of the queries from the set $Q$ that are prone to error. We assume this knowledge in all our experiments in this section. Below, we show the procedure adopted to simulate the error model,
\begin{enumerate}
\item
Select the fraction $\nu$ of the $N$ queries that are prone to error
\item
Generate $e \in \{0,\cdots,\epsilon'\}$ according to the selected probability model
\item
Choose $e$ queries from the above $N \nu$ set of queries
\item
Flip the object responses of these $e$ queries in the true object
\end{enumerate}

\begin{figure}[!t]
\begin{minipage}[b]{0.5\linewidth}
\begin{center}
 \includegraphics[width = 1.1\textwidth]{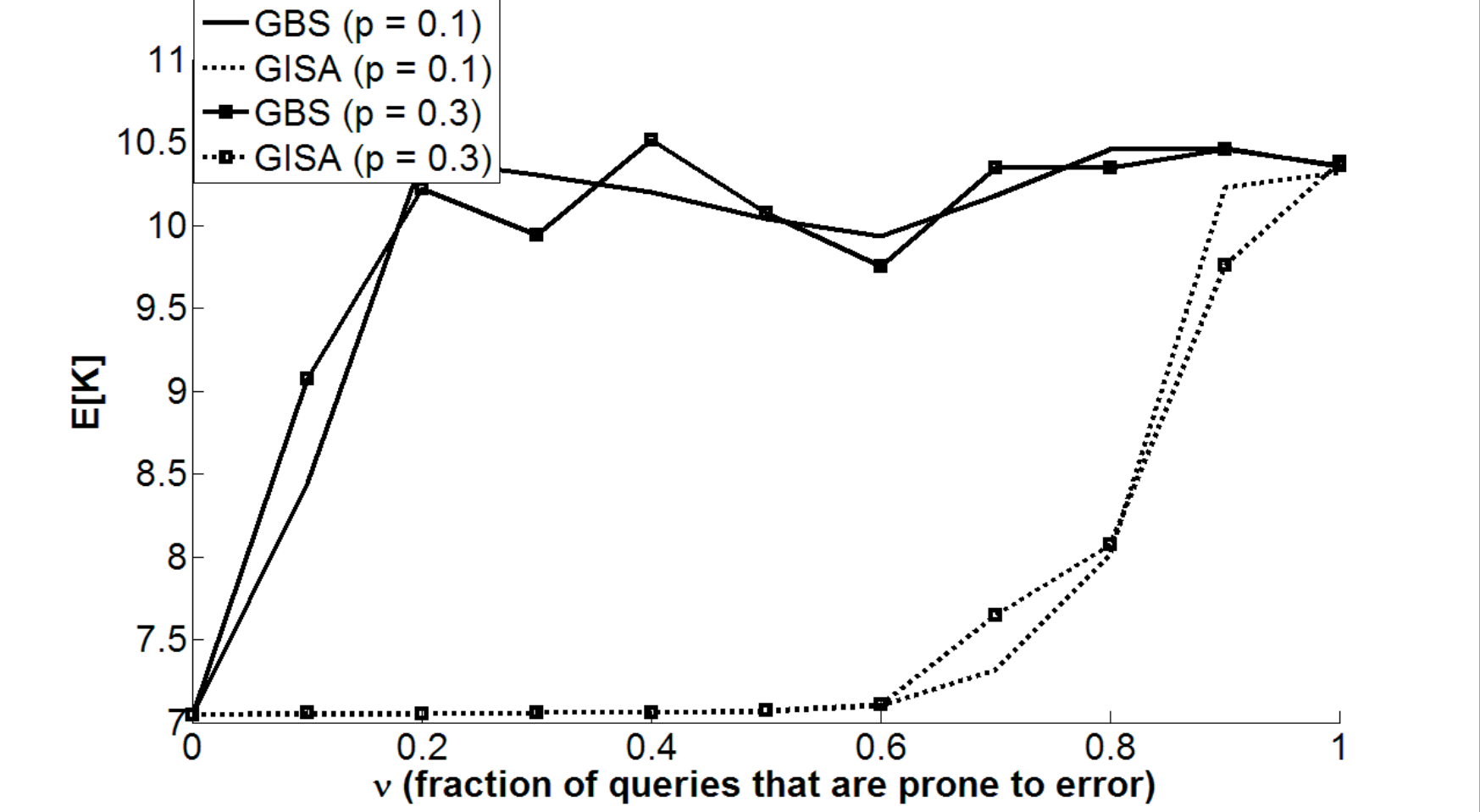}
\end{center}
\caption{\small \sl Comparison between the performance of GBS and GISA in identifying the true object in the presence of restricted persistent noise under probability model $2$}
\label{fig:WISER noise model2a}
\end{minipage}
\hspace{0.35cm}
\begin{minipage}[b]{0.5\linewidth}
\begin{center}
  \includegraphics[width = 1.1\textwidth]{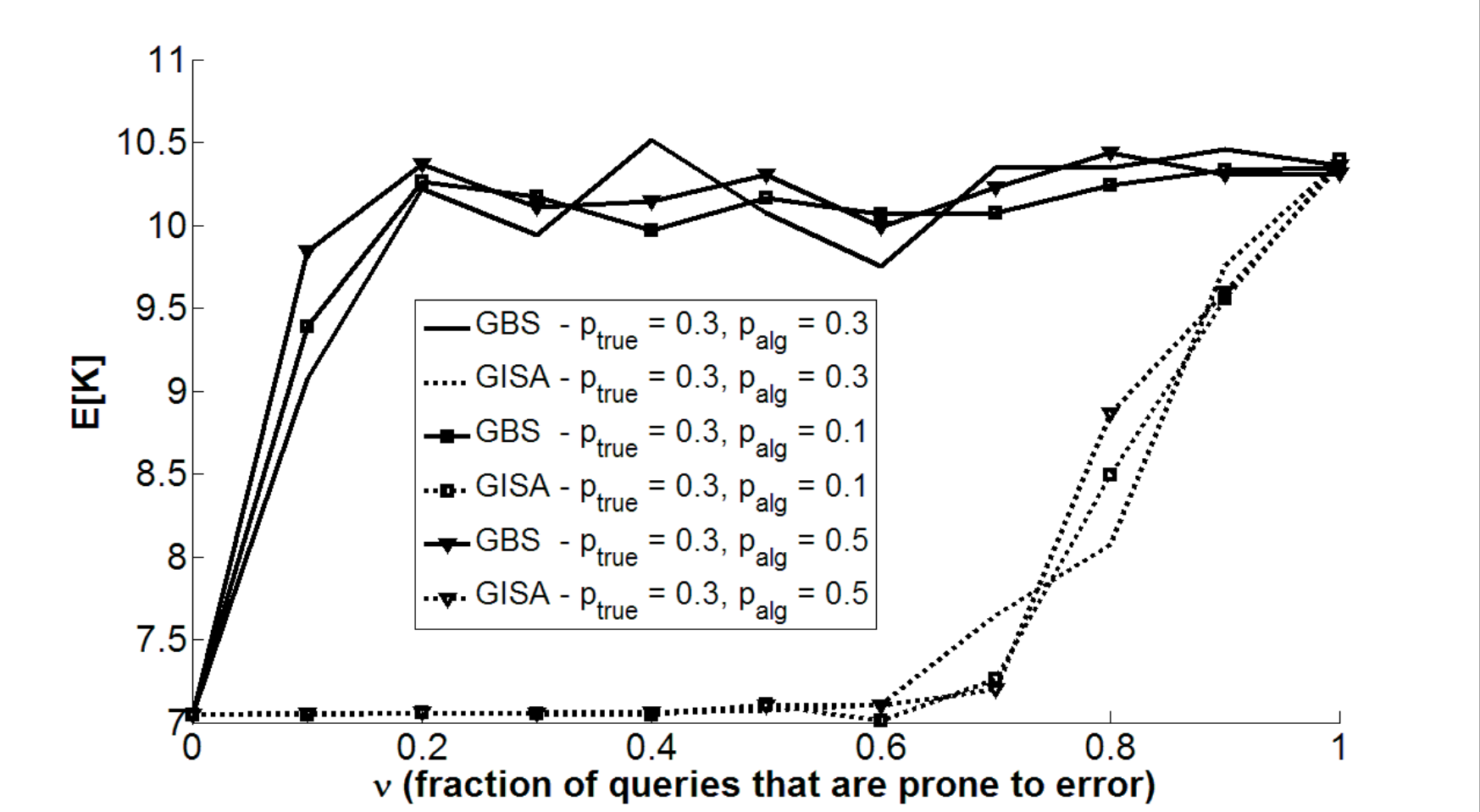}
\end{center}
\caption{\small \sl Comparison between the performance of GBS and GISA in probability model $2$ in the presence of discrepancies between the true value of $p$, $p_{true}$ and the value used in the algorithm $p_{alg}$}
\label{fig:WISER noise model2b}
\end{minipage}
\end{figure}

We compare the performance of GBS and GISA in a subset of the WISER database consisting of $131$ toxic chemicals and $79$ symptom queries with $\epsilon = 2$. Figure \ref{fig:WISER noise model1} shows the expected number of queries required by GBS and GISA to identify the true object in the presence of a maximum of $\epsilon$ persistent errors for different values of $\nu$, using probability model 1. Figure \ref{fig:WISER noise model2a} shows the same for different values of $p$ using probability model $2$. Note that except for the extreme cases where $\nu = 0$ and $\nu = 1$, GISA has great improvement over GBS. When $\nu = 0,1$, GBS and GISA reduce to the same algorithm.

Also, note that in probability model $2$, the algorithms requires the knowledge of $p$ as $\cPt$ depends on $p$. Though this probability can be estimated with the help of external knowledge sources beyond the database such as domain experts, user surveys or by analyzing past query logs, the estimated value of $p$ can vary slightly from its true value. Hence, we tested the sensitivity of these two algorithms to error in the value of $p$ and noted that there is not much change in their performance to discrepancies in the value of $p$ as shown in Figure \ref{fig:WISER noise model2b}.

\section{Conclusions and Future work}
In this paper, we developed algorithms that broaden existing methods for query learning to incorporate factors that are specific to a given task and environment. These algorithms are greedy algorithms derived in a common, principled framework based on a generalization of Shannon-Fano coding to group-based query learning. While our running example has been toxic chemical identification, the methods presented are applicable to a much broader class of applications, such as other forms of emergency response, fault diagnosis, network failure diagnosis or Internet based data search.
 
In a series of experiments on synthetic data and a toxic chemical database, we demonstrated the effectiveness of our algorithms relative to the standard splitting algorithm, also known as generalized binary search (GBS), which is the most commonly studied algorithm for query learning.  In some settings, our algorithms outperform GBS by drastic amounts. Furthermore, in the case of group identification, we have described a simple visualization (see Figure \ref{fig:WISER gamma parameters}), based on the underlying data matrix, that explains how much can be gained from GISA, our group identification algorithm. That is, it offers a picture of how much GISA will improve upon GBS without running either algorithm.

While this work is a step towards making query learning algorithms better suited to real-world identification tasks, there are many other issues that deserve to be examined in future work. These include challenges such as multiple objects present, probabilities of query response or query noise, or user confidence. In query learning with persistent noise, our approach can only recover from a restricted number of query errors, depending on the minimum Hamming distance between objects. While this assumption is required if we desire unique identification, it would be interested to loosen this assumption by pursuing a slightly less ambitious goal. Additionally, instead of minimizing the expected number of queries required for object/group identification, it would be valuable to develop a similar framework that minimizes the number of queries in the worst case, thereby eliminating dependence on the prior probabilities. Finally, it seems plausible that performance results like those proved in \cite{dasgupta} might also be possible for group-based query learning.

\section*{Appendix $\mathrm{I}$ - Proofs}
\subsection{Proof of Theorem \ref{thm:impurity based}}
\label{sec:proof1}
Let $\T_a$ denote a subtree from any node `$a$' in the tree $\T$ and let $\cL_a$ denote the set of leaf nodes in this subtree. Then, let $\mu_a$ denote the expected depth of the leaf nodes in this subtree, given by
\begin{eqnarray*}
\mu_a & = & \sum_{j \in \cL_a} \frac{\pi_{\Theta_j}}{\pi_{\Theta_a}} d_j^a
\end{eqnarray*}
where $d_j^a$ corresponds to the depth of leaf node $j$ in the subtree $\T_a$, and let $H_a$ denote the entropy of the probability distribution of the classes at the root node of the subtree $\T_a$, i.e. 
\begin{eqnarray*}
H_a = - \sum_{i=1}^m \frac{\pi_{\Theta_a^i}}{\pi_{\Theta_a}} \log \frac{\pi_{\Theta_a^i}}{\pi_{\Theta_a}}
\end{eqnarray*}
Now, we show using induction that for any subtree $\T_a$ in the tree $\T$, the following relation holds
\begin{eqnarray*} 
&\pi_{\Theta_a} \mu_a - \pi_{\Theta_a} H_a = \sum_{s \in \cI_a} \pi_{\Theta_{s}} \left [1 - H(\rho_s) + \sum_{i=1}^m \frac{\pi_{\Theta_s^i}}{\pi_{\Theta_s}} H(\rho_s^i) \right ] - \sum_{s \in \cL_a} \pi_{\Theta_s}I(\Theta_s)&
\end{eqnarray*}
where $\cI_a, \cL_a$ denotes the set of internal nodes and the set of leaf nodes in the subtree $\T_a$ respectively. 

The relation holds trivially for any subtree rooted at a leaf node of the tree $\T$ with both the left hand side and the right hand side of the expression equal to $- \pi_{\Theta_a}I(\Theta_a)$ (Note from (\ref{eq:entropy measure}) that $I(\Theta_a) = H_a$). Now, assume the above relation holds for the subtrees rooted at the left and right child nodes of node `$a$'. Then, using Lemma \ref{lem:lemma group identification} we have
\begin{eqnarray*}
\pi_{\Theta_a} [\mu_a - H_a] & = & \pi_{\Theta_{l(a)}} [\mu_{l(a)} - H_{l(a)}] + \pi_{\Theta_{r(a)}} [\mu_{r(a)} - H_{r(a)}] + \pi_{\Theta_a} \left [1 - H(\rho_a) + \sum_{i=1}^m \frac{\pi_{\Theta_a^i}}{\pi_{\Theta_a}} H(\rho_a^i) \right ] \\
& = & \sum_{s \in \cI_{l(a)}} \pi_{\Theta_{s}} \left [1 - H(\rho_s) + \sum_{i=1}^m \frac{\pi_{\Theta_s^i}}{\pi_{\Theta_s}} H(\rho_s^i) \right ] - \sum_{s \in \cL_{l(a)}} \pi_{\Theta_s}I(\Theta_s) \\
& & + \sum_{s \in \cI_{r(a)}} \pi_{\Theta_{s}} \left [1 - H(\rho_s) + \sum_{i=1}^m \frac{\pi_{\Theta_s^i}}{\pi_{\Theta_s}} H(\rho_s^i) \right ] - \sum_{s \in \cL_{r(a)}} \pi_{\Theta_s}I(\Theta_s) \\
& & +  \ \ \pi_{\Theta_a} \left [1 - H(\rho_a) + \sum_{i=1}^m \frac{\pi_{\Theta_a^i}}{\pi_{\Theta_a}} H(\rho_a^i) \right ] \\
& = & \sum_{s \in \cI_a} \pi_{\Theta_{s}} \left [1 - H(\rho_s) + \sum_{i=1}^m \frac{\pi_{\Theta_s^i}}{\pi_{\Theta_s}} H(\rho_s^i) \right ] - \sum_{s \in \cL_{a}} \pi_{\Theta_s}I(\Theta_s) 
\end{eqnarray*} 
thereby completing the induction. Finally, the result follows by applying the relation to the tree $\T$ whose probability mass at the root node, $\pi_{\Theta_a} = 1$.

\begin{lemma}
\label{lem:lemma group identification}
\begin{eqnarray*}
\pi_{\Theta_{a}} [\mu_a - H_a] & = & \pi_{\Theta_{l(a)}} [\mu_{l(a)} - H_{l(a)}] + \pi_{\Theta_{r(a)}} [\mu_{r(a)} - H_{r(a)}] + \pi_{\Theta_a} \left [1 - H(\rho_a) + \sum_{i=1}^m \frac{\pi_{\Theta_a^i}}{\pi_{\Theta_a}} H(\rho_a^i) \right ]
\end{eqnarray*}
\end{lemma}
\begin{proof}
We first note that $\pi_{\Theta_{a}}\mu_a$ for a subtree $\T_a$ can be decomposed as
\begin{eqnarray}
\label{eq:mu decomposition group identification}
\pi_{\Theta_{a}}\mu_a & = & \sum_{j \in \cL_a} \pi_{\Theta_j} d_j^a \nonumber \\
 &= & \sum_{j \in \cL_{l(a)}} \pi_{\Theta_j} d_j^a + \sum_{j \in \cL_{r(a)}} \pi_{\Theta_j}d_j^a \nonumber \\ 
& = & \sum_{j \in \cL_{l(a)}} \pi_{\Theta_j} (d_j^a - 1) + \sum_{j \in \cL_{r(a)}} \pi_{\Theta_j} (d_j^a-1) + \sum_{j \in \cL_a} \pi_{\Theta_j} \nonumber \\ 
& = & \pi_{\Theta_{l(a)}} \mu_{l(a)} + \pi_{\Theta_{r(a)}} \mu_{r(a)} + \pi_{\Theta_a}
\end{eqnarray}
Similarly, $\pi_{\Theta_a}H_a$ can be decomposed as
\begin{eqnarray}
\label{eq:entropy decomposition group identification}
\pi_{\Theta_a}H_a & = & \sum_{i=1}^m \pi_{\Theta_a^i} \log \frac{\pi_{\Theta_a}}{\pi_{\Theta_a^i}}  \nonumber \\
& = & \sum_{i=1}^m \pi_{\Theta_{l(a)}^i} \log \frac{\pi_{\Theta_a}}{\pi_{\Theta_a^i}} + \sum_{i=1}^m \pi_{\Theta_{r(a)}^i} \log \frac{\pi_{\Theta_a}}{\pi_{\Theta_a^i}} \nonumber \\
& = & \sum_{i=1}^m \pi_{\Theta_{l(a)}^i} \log \frac{\pi_{\Theta_{l(a)}}}{\pi_{\Theta_{l(a)}^i}} + \sum_{i=1}^m \pi_{\Theta_{l(a)}^i} \log \frac{\pi_{\Theta_{l(a)}^i}}{\pi_{\Theta_a^i}} + \sum_{i=1}^m \pi_{\Theta_{l(a)}^i} \log \frac{\pi_{\Theta_a}}{\pi_{\Theta_{l(a)}}} \nonumber \\
& & + \sum_{i=1}^m \pi_{\Theta_{r(a)}^i} \log \frac{\pi_{\Theta_{r(a)}}}{\pi_{\Theta_{r(a)}^i}} + \sum_{i=1}^m \pi_{\Theta_{r(a)}^i} \log \frac{\pi_{\Theta_{r(a)}^i}}{\pi_{\Theta_a^i}} + \sum_{i=1}^m \pi_{\Theta_{r(a)}^i} \log \frac{\pi_{\Theta_a}}{\pi_{\Theta_{r(a)}}} \nonumber \\
& = & \pi_{\Theta_{l(a)}} H_{l(a)} + \pi_{\Theta_{r(a)}} H_{r(a)} - \sum_{i=1}^m \left [ \pi_{\Theta_{l(a)}^i} \log \frac{\pi_{\Theta_{a}^i}}{\pi_{\Theta_{l(a)}^i}} + \pi_{\Theta_{r(a)}^i} \log \frac{\pi_{\Theta_{a}^i}}{\pi_{\Theta_{r(a)}^i}} \right ] \nonumber \\
 & & + \left [ \pi_{\Theta_{l(a)}} \log \frac{\pi_{\Theta_a}}{\pi_{\Theta_{l(a)}}} + \pi_{\Theta_{r(a)}} \log \frac{\pi_{\Theta_a}}{\pi_{\Theta_{r(a)}}} \right ] \nonumber \\
& = & \pi_{\Theta_{l(a)}} H_{l(a)} + \pi_{\Theta_{r(a)}} H_{r(a)} - \sum_{i=1}^m \pi_{\Theta_a^i} H(\rho_a^i) + \pi_{\Theta_a} H(\rho_a)
\end{eqnarray}
The result follows from (\ref{eq:mu decomposition group identification}) and (\ref{eq:entropy decomposition group identification}) above.
\end{proof}

\subsection{Proof of Theorem \ref{thm:connection to impurity based DT induction}}
\label{sec:proof2}
From relation \eqref{eq:entropy decomposition group identification} in Lemma \ref{lem:lemma group identification}, we have
\begin{align*}
H_a - \left [\frac{\pi_{\Theta_{l(a)}}}{\pi_{\Theta_a}}H_{l(a)} + \frac{\pi_{\Theta_{r(a)}}}{\pi_{\Theta_a}}H_{r(a)} \right ] & = - \left [ - H(\rho_a ) + \sum_{i=1}^m \frac{\pi_{\Theta_a^i}}{\pi_{\Theta_a}} H(\rho_a^i) \right ] 
\end{align*}
Thus, maximizing the impurity based objective function with entropy function as the impurity function is equivalent to minimizing the cost function $C_a := 1 - H(\rho_a ) + \sum_{i=1}^m \frac{\pi_{\Theta_a^i}}{\pi_{\Theta_a}} H(\rho_a^i)$

\subsection{Proof of Theorem \ref{thm:group identification group queries}}
\label{sec:proof3}
Let $\T_a$ denote a subtree from any node `$a$' in the tree $\T$ and let $\cL_a$ denote the set of leaf nodes in this subtree. Then, let $\mu_a$ denote the expected number of queries required to identify the group of an object terminating in a leaf node of this subtree, given by 
\begin{eqnarray*}
\mu_a & = & \sum_{j \in \cL_a} \frac{\pi_{\Theta_j}}{\pi_{\Theta_{a}}} \p_j^a d_j^a
\end{eqnarray*}
where $d_j^a,\p_j^a$ denotes the depth of leaf node $j$ in the subtree $T_a$ and the probability of reaching that leaf node given $\theta \in \Theta_j$, respectively, and let $H_a$ denote the entropy of the probability distribution of the object groups at the root node of this subtree, i.e.
\begin{eqnarray*}
H_a = - \sum_{i=1}^m \frac{\pi_{\Theta_a^i}}{\pi_{\Theta_a}} \log \frac{\pi_{\Theta_a^i}}{\pi_{\Theta_a}}
\end{eqnarray*}
Now, we show using induction that for any subtree $\T_a$ in the tree $\T$, the following relation holds
\begin{eqnarray*}
\pi_{\Theta_{a}}\mu_a - \pi_{\Theta_a}H_a = \sum_{s \in \cI_a} \p_s^a \pi_{\Theta_{s}} \left \{ 1 - \sum_{q \in Q^{z_s} } p_{z_s}(q)\left [ H(\rho_{s}(q) ) - \sum_{i=1}^m \frac{\pi_{\Theta_{s}^i}}{\pi_{\Theta_s}} H(\rho_{s}^i(q)) \right ] \right \}
\end{eqnarray*}
where $\cI_a$ denotes the set of internal nodes in the subtree $\T_a$.

The relation holds trivially for any subtree rooted at a leaf node of the tree $\T$ with both the left hand side and the right hand side of the expression being equal to $0$. Now, assume the above relation holds for all subtrees rooted at the child nodes of node `$a$'. Note that node `$a$' has a set of left and right child nodes, each set corresponding to one query from the query group selected at that node. Then, using the decomposition in Lemma \ref{lem:lemma group identification} on each query from this query group, we have
\begin{align*}
1\cdot\pi_{\Theta_a} [ \mu_a - H_a] &= \sum_{q \in Q^{z_a}} p_{z_a}(q) \pi_{\Theta_a} [\mu_a - H_a ] \\
 & =  \sum_{q \in Q^{z_a}} p_{z_a}(q) \left \{\pi_{\Theta_{l^q(a)}} [\mu_{l^q(a)} -  H_{l^q(a)}] + \pi_{\Theta_{r^q(a)}}[ \mu_{r^q(a)} - H_{r^q(a)}] \right. \\
&  \ \ \quad  + \left. \pi_{\Theta_a} \left [ 1 -  H(\rho_a(q)) - \sum_{i=1}^m \frac{\pi_{\Theta_a^i}}{\pi_{\Theta_a}} H(\rho_a^i(q)) \right ] \right \} \\
& = \sum_{q \in Q^{z_a}} p_{z_a}(q) \left \{\pi_{\Theta_{l^q(a)}} [\mu_{l^q(a)} -  H_{l^q(a)}] + \pi_{\Theta_{r^q(a)}}[ \mu_{r^q(a)} - H_{r^q(a)}] \right \} \\
&  \ \ \quad  + \pi_{\Theta_a} \left \{ 1 - \sum_{q \in Q^{z_a}}p_{z_a}(q) \left [ H(\rho_a(q)) - \sum_{i=1}^m \frac{\pi_{\Theta_a^i}}{\pi_{\Theta_a}} H(\rho_a^i(q)) \right ] \right \}
\end{align*} 
where $l^q(a), r^q(a)$ correspond to the left and right child of node `$a$' when query $q$ is chosen from the query group and $\mu_{l^q(a)}, \pi_{\Theta_{l^q(a)}}, H_{l^q(a)}$ correspond to the expected depth of a leaf node in the subtree $\T_{l^q(a)}$, probability mass of the objects at the root node of this subtree, and the entropy of the probability distribution of the objects at the root node of this subtree respectively. Now, using the induction hypothesis, we get

\begin{align*}
\pi_{\Theta_a} \mu_a - \pi_{\Theta_a}H_a & = \sum_{q \in Q^{z_a}} p_{z_a}(q) \left \{ \sum_{s \in \cI_{l^q(a)}} \p_s^{l^q(a)} \pi_{\Theta_{s}} \left [ 1 - \sum_{q \in Q^{z_s} } p_{z_s}(q)\left ( H(\rho_{s}(q) ) - \sum_{i=1}^m \frac{\pi_{\Theta_{s}^i}}{\pi_{\Theta_s}} H(\rho_{s}^i(q)) \right ) \right ] \right \} \\
& \ \ \quad + \sum_{q \in Q^{z_a}} p_{z_a}(q) \left \{ \sum_{s \in \cI_{r^q(a)}} \p_s^{r^q(a)} \pi_{\Theta_{s}} \left [ 1 - \sum_{q \in Q^{z_s} } p_{z_s}(q)\left ( H(\rho_{s}(q) ) - \sum_{i=1}^m \frac{\pi_{\Theta_{s}^i}}{\pi_{\Theta_s}} H(\rho_{s}^i(q)) \right ) \right ] \right \} \\
& \ \ \quad + \pi_{\Theta_a} \left \{ 1 - \sum_{q \in Q^{z_a}}p_{z_a}(q) \left [ H(\rho_a(q)) - \sum_{i=1}^m \frac{\pi_{\Theta_a^i}}{\pi_{\Theta_a}} H(\rho_a^i(q)) \right ] \right \} \\
& =  \sum_{s \in \cI_a} \p_s^a \pi_{\Theta_{s}} \left \{ 1 - \sum_{q \in Q^{z_s} } p_{z_s}(q)\left [ H(\rho_{s}(q) ) - \sum_{i=1}^m \frac{\pi_{\Theta_{s}^i}}{\pi_{\Theta_s}} H(\rho_{s}^i(q)) \right ] \right \}
\end{align*}
thereby completing the induction. Finally, the result follows by applying the relation to the subtree rooted at the root node of $\T$, whose probability mass $\pi_{\Theta_a} = 1$.

\section*{Appendix $\mathrm{II}$}
\subsection*{Reduction factor calculation in the persistent noise model}
\label{sec:reduction factor calculation}
At any internal node $a \in \cI$ in a tree, let $\delta_i^a$ denote the Hamming distance between the query responses up to this internal node ($Q_a$) and the true responses of object $\theta_i$ to those queries. Also, let $n_a$ denote the number of queries from the set of $N\nu$ queries (that were prone to error) in the set $Q \setminus Q_a$ and for a query $q \in Q \setminus Q_a$, denote by $b_i(q)$ the binary response of object $\theta_i$ to that query. Denote by the set $I^a = \{i: \delta_i^a \leq \epsilon'\}$, the object groups with non-zero number of objects at this internal node. All the formulas below come from routine calculations based on probability model $2$.  

For a query $q \in Q \setminus Q_a$, that is not prone to error, the reduction factor and the group reduction factors generated by choosing that query at node `$a$' are as follows. The group reduction factor of any group $i \in I^a$ is equal to $1$ and the reduction factor is given by
\begin{eqnarray*}
\rho_a = \frac{\max \left \{ \underset{i \in I_0^{a}}{\operatorname{\sum}} \pi_i \left [\overset{\tau_i^{a}}{\underset{e=0}{\operatorname{\sum}}} {n_a \choose e} p^{e+\delta_i^a}(1-p)^{N\nu-e-\delta_i^a} \right ], \underset{i \in I_1^{a}}{\operatorname{\sum}} \pi_i \left [\overset{\tau_i^{a}}{\underset{e=0}{\operatorname{\sum}}} {n_a \choose e} p^{e+\delta_i^a}(1-p)^{N\nu-e-\delta_i^a} \right ] \right \}}{\underset{i \in I_0^a \bigcap I_1^a}{\operatorname{\sum}} \pi_i \left [\overset{\tau_i^{a}}{\underset{e=0}{\operatorname{\sum}}} {n_a \choose e} p^{e+\delta_i^a}(1-p)^{N\nu-e-\delta_i^a} \right ]} 
\end{eqnarray*} 
where $I_0^a = \{ i \in I^a: b_i(q) = 0 \}$, $I_1^a = \{ i \in I^a: b_i(q) = 1 \}$ and $\tau_i^a = \min (n_a,\epsilon' - \delta_i^a)$.

In addition, for a query $q \in Q \setminus Q_a$ that is prone to error, denote by $\delta_i^{l(a)},\delta_i^{r(a)}$ the Hamming distance between the user responses to queries up to the left and right child node of node `$a$' with query $q$ chosen at node `$a$', and the true responses of object $\theta_i$ to those queries. In particular, $\delta_i^{l(a)} = \delta_i^a + |b_i(q) - 0|$ and $\delta_i^{r(a)} = \delta_i^a + |b_i(q) - 1|$. Then, the reduction factor and the group reduction factors generated by choosing this query at node `$a$' are as follows. The group reduction factor of a group $i \in I^a$ whose $\delta_i^a = \epsilon'$ is equal to $1$ and that of a group whose $\delta_i^a < \epsilon'$ is given by
\begin{eqnarray*}
\rho_a^i = \frac{\max \left \{ \overset{\tau_i^{l(a)}}{\underset{e=0}{\operatorname{\sum}}} {n_a - 1 \choose e} p^{e+\delta_i^{l(a)}}(1-p)^{N\nu-e-\delta_i^{l(a)}}, \overset{\tau_i^{r(a)}}{\underset{e=0}{\operatorname{\sum}}} {n_a - 1 \choose e} p^{e+\delta_i^{r(a)}}(1-p)^{N\nu-e-\delta_i^{r(a)}}   \right \}}{\overset{\tau_i^{a}}{\underset{e=0}{\operatorname{\sum}}} {n_a \choose e} p^{e+\delta_i^a}(1-p)^{N\nu-e-\delta_i^a}}
\end{eqnarray*}
where $\tau_i^{l(a)} = \min (n_a - 1,\epsilon' - \delta_i^{l(a)})$ and $\tau_i^{r(a)} = \min (n_a - 1,\epsilon' - \delta_i^{r(a)})$, and the reduction factor is given by
\begin{eqnarray*}
\rho_a = \frac{\max \left \{ \underset{i \in I^{l(a)}}{\operatorname{\sum}} \pi_i \left [ \overset{\tau_i^{l(a)}}{\underset{e=0}{\operatorname{\sum}}} {n_a - 1 \choose e} p^{e+\delta_i^{l(a)}}(1-p)^{N\nu-e-\delta_i^{l(a)}} \right ], \underset{i \in I^{r(a)}} {\operatorname{\sum}} \pi_i \left [ \overset{\tau_i^{r(a)}}{\underset{e=0}{\operatorname{\sum}}} {n_a - 1 \choose e} p^{e+\delta_i^{r(a)}}(1-p)^{N\nu-e-\delta_i^{r(a)}} \right ] \right \}}{\underset{i \in I^{a}}{\operatorname{\sum}} \pi_i \left [ \overset{\tau_i^{a}}{\underset{e=0}{\operatorname{\sum}}} {n_a  \choose e} p^{e+\delta_i^{r(a)}}(1-p)^{N\nu-e-\delta_i^{r(a)}} \right ] }
\end{eqnarray*}

\section*{Acknowledgment}
The authors would like to thank R. Nowak for helpful feedback and S. S. Pradhan for insightful discussions during the initial phase of this research. Also, the authors thank A. Ganesan, R. Richardson, P. Saxman, G. Vallabha and C. Weber for their contributions, and C. Weber for partitioning the WISER chemicals into groups. This work was supported in part by NSF Grant CCF-0830490 and NIH Grant $\#$UL1RR024986.

\bibliographystyle{IEEEtran}
\bibliography{ref}


\end{document}

%% file: commands.tex
\newcommand{\beas}{\begin{eqnarray*}}
\newcommand{\eeas}{\end{eqnarray*}}

{\endtrivlist}